\documentclass[journal]{IEEEtran}
\usepackage{cite}

\usepackage[linesnumbered, ruled]{algorithm2e}
\usepackage[breaklinks]{hyperref}
\usepackage{xspace} 
\usepackage{mathrsfs}
\usepackage{amsmath,amsthm,amssymb,amsfonts}
\usepackage{booktabs}
\usepackage{makecell}
\usepackage{array}
\usepackage{bm}
\usepackage[caption=false,font=normalsize,labelfont=sf,textfont=sf]{subfig}
\usepackage{textcomp}
\usepackage{stfloats}
\usepackage{url}
\usepackage{verbatim}
\usepackage{graphicx}
\usepackage{cite}
\usepackage{array}
\usepackage{tabularx}
\usepackage{multirow}
\usepackage{bbding}
\usepackage{pifont}
\usepackage{wasysym}
\usepackage{color}
\usepackage{caption}
\usepackage{subcaption}
\usepackage{multicol}
\hypersetup{hidelinks,
	colorlinks=true,
	allcolors=black,
	pdfstartview=Fit,
	breaklinks=true}
\usepackage{amsmath}
\usepackage{amssymb}
\DeclareMathAlphabet\mathbfcal{OMS}{cmsy}{b}{n}
\usepackage{booktabs} 
\usepackage{multirow}
\usepackage{pifont}

\newtheorem{theorem}{Theorem}
\newtheorem{lemma}{Lemma}

\theoremstyle{plain}
\newtheorem{definition}{Definition}

\definecolor{mygray}{gray}{.91}
\usepackage[table]{xcolor}
\begin{document}
%
% paper title
% Titles are generally capitalized except for words such as a, an, and, as,
% at, but, by, for, in, nor, of, on, or, the, to and up, which are usually
% not capitalized unless they are the first or last word of the title.
% Linebreaks \\ can be used within to get better formatting as desired.
% Do not put math or special symbols in the title.
\title{SWIFT: A Small-World Interaction Framework for Flow-Aware Trajectory Prediction in Autonomous Driving}

\author{Chengyue~Wang,
        Bin~Rao,
        Haicheng~Liao,
        Bonan~Wang,
        Chengzhong~Xu, ~\IEEEmembership{Fellow, ~IEEE}
        and Zhenning Li$^{\dag}$
        % <-this % stops a space
\thanks{\dag\,Corresponding author. E-mails: zhenningli@um.edu.mo} \thanks{Chengyue Wang, Bin Rao, Haicheng Liao, Bonan Wang, Chengzhong Xu, and Zhenning Li are with the State Key Laboratory of Internet of Things for Smart City, University of Macau, Macau. Chengyue Wang and Zhenning Li are also affiliated with the Department of Civil and Environmental Engineering, University of Macau, Macau. }% <-this % stops a space
\thanks{This work was supported by the Science and Technology Development Fund of Macau [0122/2024/RIB2, 0215/2024/AGJ, 001/2024/SKL], the Research Services and Knowledge Transfer Office, University of Macau [SRG2023-00037-IOTSC, MYRG-GRG2024-00284-IOTSC], the Shenzhen-Hong Kong-Macau Science and Technology Program Category C [SGDX20230821095159012], the Science and Technology Planning Project of Guangdong [2025A0505010016], National Natural Science Foundation of China [52572354], the State Key Lab of Intelligent Transportation System [2024-B001], and the Jiangsu Provincial Science and Technology Program [BZ2024055].}}
% <-this % stops a space

%\thanks{Manuscript received April 19, 2005; revised August 26, 2015.}}

% note the % following the last \IEEEmembership and also \thanks - 
% these prevent an unwanted space from occurring between the last author name
% and the end of the author line. i.e., if you had this:
% 
% \author{....lastname \thanks{...} \thanks{...} }
%                     ^------------^------------^----Do not want these spaces!
%
% a space would be appended to the last name and could cause every name on that
% line to be shifted left slightly. This is one of those "LaTeX things". For
% instance, "\textbf{A} \textbf{B}" will typeset as "A B" not "AB". To get
% "AB" then you have to do: "\textbf{A}\textbf{B}"
% \thanks is no different in this regard, so shield the last } of each \thanks
% that ends a line with a % and do not let a space in before the next \thanks.
% Spaces after \IEEEmembership other than the last one are OK (and needed) as
% you are supposed to have spaces between the names. For what it is worth,
% this is a minor point as most people would not even notice if the said evil
% space somehow managed to creep in.

% The paper headers IEEE Transactions on Intelligent Transportation Systems
\markboth{Journal of IEEE Transactions on Pattern Analysis and Machine Intelligence, January~2026}%
{Wang \MakeLowercase{\textit{et al.}}: Bare Demo of IEEEtran.cls for IEEE Journals}
% The only time the second header will appear is for the odd numbered pages
% after the title page when using the twoside option.
% 
% *** Note that you probably will NOT want to include the author's ***
% *** name in the headers of peer review papers.                   ***
% You can use \ifCLASSOPTIONpeerreview for conditional compilation here if
% you desire.

% If you want to put a publisher's ID mark on the page you can do it like
% this:
%\IEEEpubid{0000--0000/00\$00.00~\copyright~2015 IEEE}
% Remember, if you use this you must call \IEEEpubidadjcol in the second
% column for its text to clear the IEEEpubid mark.

% use for special paper notices
%\IEEEspecialpapernotice{(Invited Paper)}

% make the title area
\maketitle

% As a general rule, do not put math, special symbols or citations
% in the abstract or keywords.
\begin{abstract}
Accurate trajectory prediction in autonomous driving hinges on modeling dynamic and context-dependent interactions among traffic agents. However, most existing approaches are purely data-driven and lack structural priors, which limits their generalization under distribution shifts. In this work, interaction modeling is revisited through the structure and dynamics of traffic networks, and SWIFT (Small-World Interaction Framework for Trajectory prediction) is proposed as a unified framework that integrates small-world networks with traffic flow theory. SWIFT introduces structural inductive biases via a Small-World Interaction Network that captures both local and global dependencies, and a Flow Regime Encoder that adapts the interaction structure to scene-level traffic states. Interaction reasoning is further enhanced through a multi-relational graph module that explicitly encodes direct and higher-order agent relationships. Extensive experiments on three real-world datasets, nuScenes, MoCAD, and NGSIM, show that SWIFT consistently outperforms strong baselines in prediction accuracy across diverse traffic regimes. Beyond accuracy, SWIFT exhibits improved generalization to unseen locations and regimes, robustness under noisy observations, and strong performance with limited training data, supporting the effectiveness of its structure-aware design.
\end{abstract}

% Note that keywords are not normally used for peerreview papers.
\begin{IEEEkeywords}
Autonomous Driving, Trajectory Prediction, Interaction Modeling, Small-world Networks
\end{IEEEkeywords}

% For peer review papers, you can put extra information on the cover
% page as needed:
% \ifCLASSOPTIONpeerreview
% \begin{center} \bfseries EDICS Category: 3-BBND \end{center}
% \fi
%
% For peerreview papers, this IEEEtran command inserts a page break and
% creates the second title. It will be ignored for other modes.
\IEEEpeerreviewmaketitle

\section{Introduction}
% The very first letter is a 2 line initial drop letter followed
% by the rest of the first word in caps.
% 
% form to use if the first word consists of a single letter:
% \IEEEPARstart{A}{demo} file is ....
% 
% form to use if you need the single drop letter followed by
% normal text (unknown if ever used by the IEEE):
% \IEEEPARstart{A}{}demo file is ....
% 
% Some journals put the first two words in caps:
% \IEEEPARstart{T}{his demo} file is ....
% 
% Here we have the typical use of a "T" for an initial drop letter
% and "HIS" in caps to complete the first word.

\IEEEPARstart{T}{rajectory} prediction is a foundational task in autonomous driving, serving as a critical input for downstream modules such as risk assessment, planning, and control~\cite{marchetti2024smemo}. It involves forecasting the future motion of surrounding traffic agents, including vehicles, pedestrians, and cyclists, based on their historical trajectories and environmental context. Recent advances in deep learning have markedly improved prediction accuracy by learning data-driven representations, as shown in Fig.~\ref{toutu}(a), of interaction patterns between traffic agents~\cite{sun2023modality, zhang2024decoupling}. Despite these advancements, fundamental challenges remain that limit the generalizability, robustness, and interpretability of current models, especially in complex traffic conditions~\cite{liu2025end,liao2025sa}.

\begin{figure}[t]
        \centering
	\includegraphics[width=\linewidth]{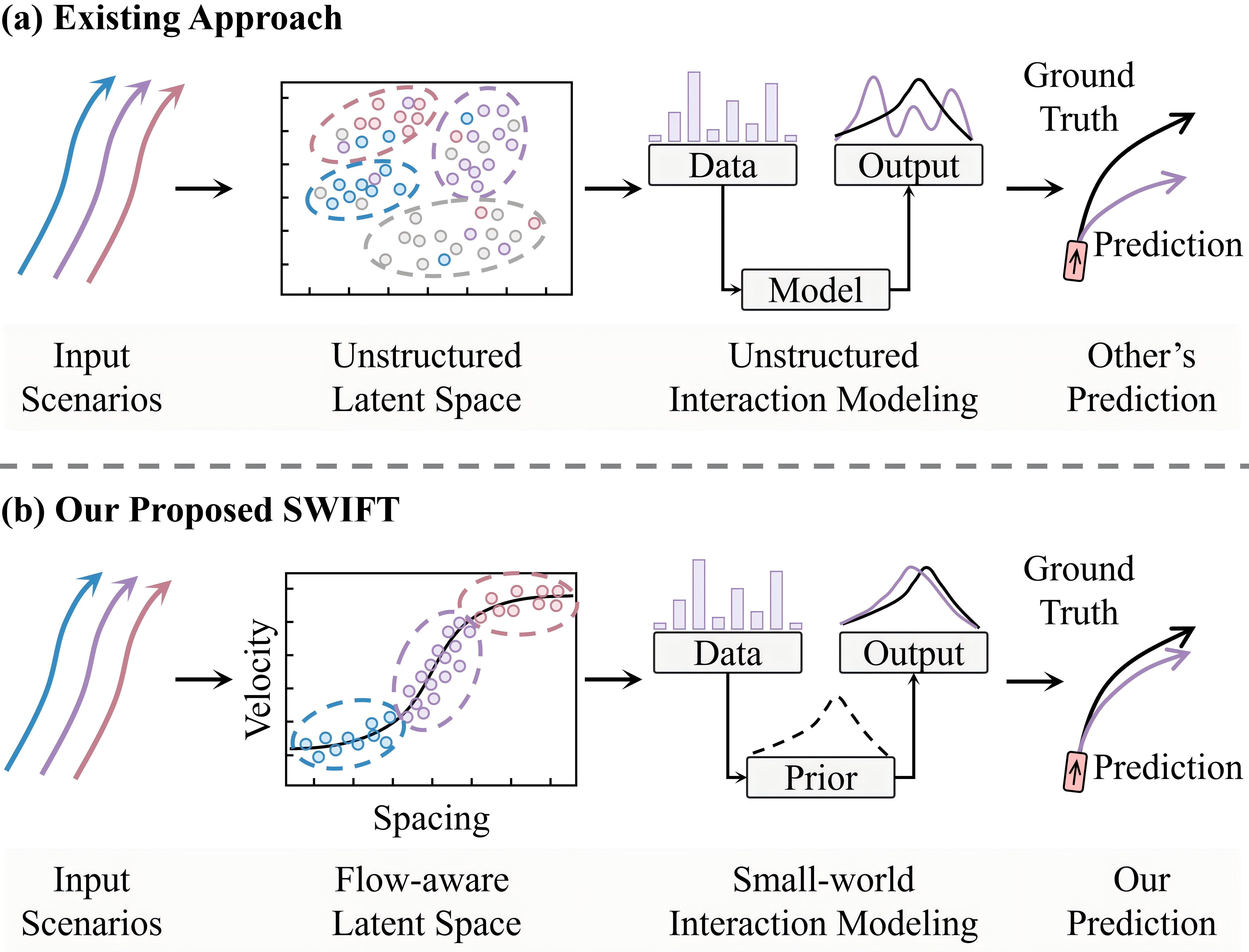}
	\caption{Conceptual comparison between (a) existing approach and (b) our proposed SWIFT. While existing approaches learn unstructured latent spaces and interactions purely from data, SWIFT introduces flow-aware latent spaces inspired by traffic flow theory and models agent interactions via a Small-World Interaction Network guided by structural priors.}
        \label{toutu}
\end{figure}

A key limitation lies in how inter-agent \textbf{interactions are modeled structurally}. Many state-of-the-art methods construct interaction graphs using proximity thresholds, predefined spatial grids, or self-attention mechanisms~\cite{du2024social,mo2024heterogeneous}. While these approaches are effective in localized, low-speed scenarios, they often fail to capture the global dependencies that naturally arise in traffic systems. For instance, on highways or in dense urban environments, a braking maneuver by a vehicle at one end of the scene may induce coordinated reactions several vehicles away. These effects, while spatially distant, are temporally and causally linked. Ignoring such global dependencies leads to fragmented reasoning and delayed hazard anticipation. Crucially, traffic systems often exhibit \textbf{small-world properties}~\cite{ghavasieh2024diversity}, which are characterized by high local clustering and short average path lengths, thereby facilitating rapid propagation of influence. However, these structural priors remain largely absent in existing data-driven interaction modeling of the trajectory prediction model.

A second limitation is that existing models typically \textbf{treat interactions as static}, or uniform across contexts, disregarding the scene-level traffic regime. However, traffic flow theory has long emphasized that agent behaviors, along with their interactions, are highly sensitive to macroscopic flow states~\cite{kerner2009introduction}. In free flow regimes, agents are loosely coupled, and small perturbations dissipate. In contrast, synchronized flow regimes are dominated by high-density coupling and reduced stability, where minor disruptions may amplify across the scene~\cite{kerner2004three}. As illustrated in Fig.~\ref{diagram}, different regimes exhibit markedly distinct patterns of interaction and system response. Models~\cite{liu2024attention, chen2022intention} that rely solely on raw kinematics or local geometry cannot differentiate among these regimes, and thus struggle to adapt their interaction structures to varying traffic dynamics.

Third, many models adopt pairwise interaction structures; they fail to account for \textbf{higher-order relational dynamics}. These occur when agents are not merely influenced by or influencing one other agent directly but are jointly affected by shared upstream context or coordinate actions as part of emergent group behavior~\cite{helbing2012social}. Such interactions are not simply the sum of individual pairs; they require reasoning about motifs, including direct, co-influencing, and co-influenced group structures, which are prevalent in complex scenes such as intersections, roundabouts, and merging zones~\cite{chen2025dstigcn, westny2023mtp}.

\begin{figure}[t]
        \centering
	\includegraphics[width=1\linewidth]{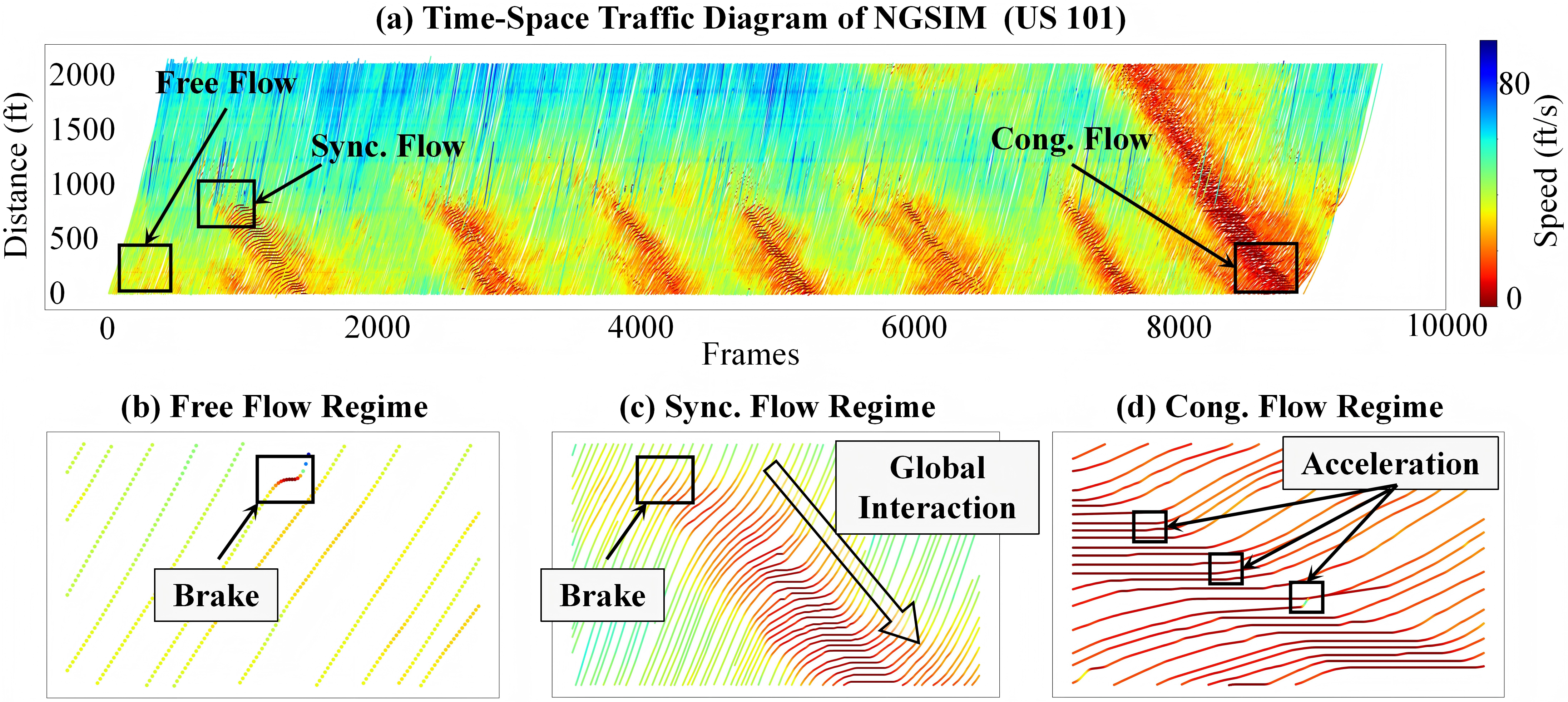}
	\caption{Spatial-temporal visualization and regime-specific decomposition of the NGSIM (US 101 7:50-8:05 AM) dataset~\cite{simulation2007us}. (a) Time-Space traffic diagram colored by speed. (b–d) Representative scenario examples sampled from each regime: (b) In the free flow regime, interactions are sparse, with brake behavior that does not alter the overall traffic dynamics. (c) In the synchronized flow regime, a single brake behavior will propagate globally, leading to large-scale traffic transitions. (d) In the congested flow regime, acceleration behavior typically has only localized effects. These examples illustrate the distinct interaction patterns within each regime.}
    \label{diagram}
\end{figure}

To address these interrelated challenges, we propose \textbf{SWIFT}, \textbf{S}mall-\textbf{W}orld \textbf{I}nteraction \textbf{F}ramework for \textbf{T}rajectory prediction, a unified model that integrates structural priors from network science with macroscopic insights from traffic flow theory, as shown in Fig.~\ref{toutu}(b). SWIFT rethinks interaction modeling as both structurally informed and context-aware: it constructs interaction graphs with small-world properties to capture both localized influence and global dependencies and adapts its modeling strategy based on traffic regime conditions (e.g., free-flow vs. congestion). By incorporating flow-aware representations and structured relational reasoning, SWIFT offers a robust, generalizable, and theoretically grounded alternative to existing data-driven paradigms. Extensive evaluations on diverse datasets, including nuScenes, MoCAD, and NGSIM, demonstrate SWIFT’s superior predictive accuracy, robustness under noise, and strong generalization, even with limited data.
Our contributions are summarized as follows:

\begin{itemize} 

\item \textbf{Small-World Interaction Network.} We design a hybrid interaction graph construction framework that enforces small-world properties, through a hybrid optimization-and-learning approach that jointly integrates explicit structural priors with data-driven refinement. This graph structure enables effective modeling of both local interactions and global interactions.

\item \textbf{Flow Regime Encoder.} We introduce a dynamic encoding mechanism that estimates macroscopic traffic statistics, such as agent density and velocity variance, and adjusts interaction graph parameters accordingly. This allows the model to adapt its connectivity and reasoning strategy to match the underlying traffic regime.

\item \textbf{Multi-Relational Interaction Reasoning.} We develop a multi-channel graph reasoning module that disentangles direct, co-influenced, and co-influencing agent relationships. This enables SWIFT to model complex group behaviors and higher-order dependencies that are often overlooked in pairwise interaction structures.

\end{itemize}

The remainder of the paper is organized as follows. Section~\ref{RelatedWork} reviews related literature in trajectory prediction, interaction modeling, and network structure in traffic systems. Section~\ref{Method} describes the SWIFT architecture and its theoretical foundations. Section~\ref{Experiment} presents extensive experimental results and ablation studies. Finally, Section~\ref{Conclusion} discusses implications and potential extensions.

\begin{figure*}[t]
        \centering
	\includegraphics[width=\textwidth]{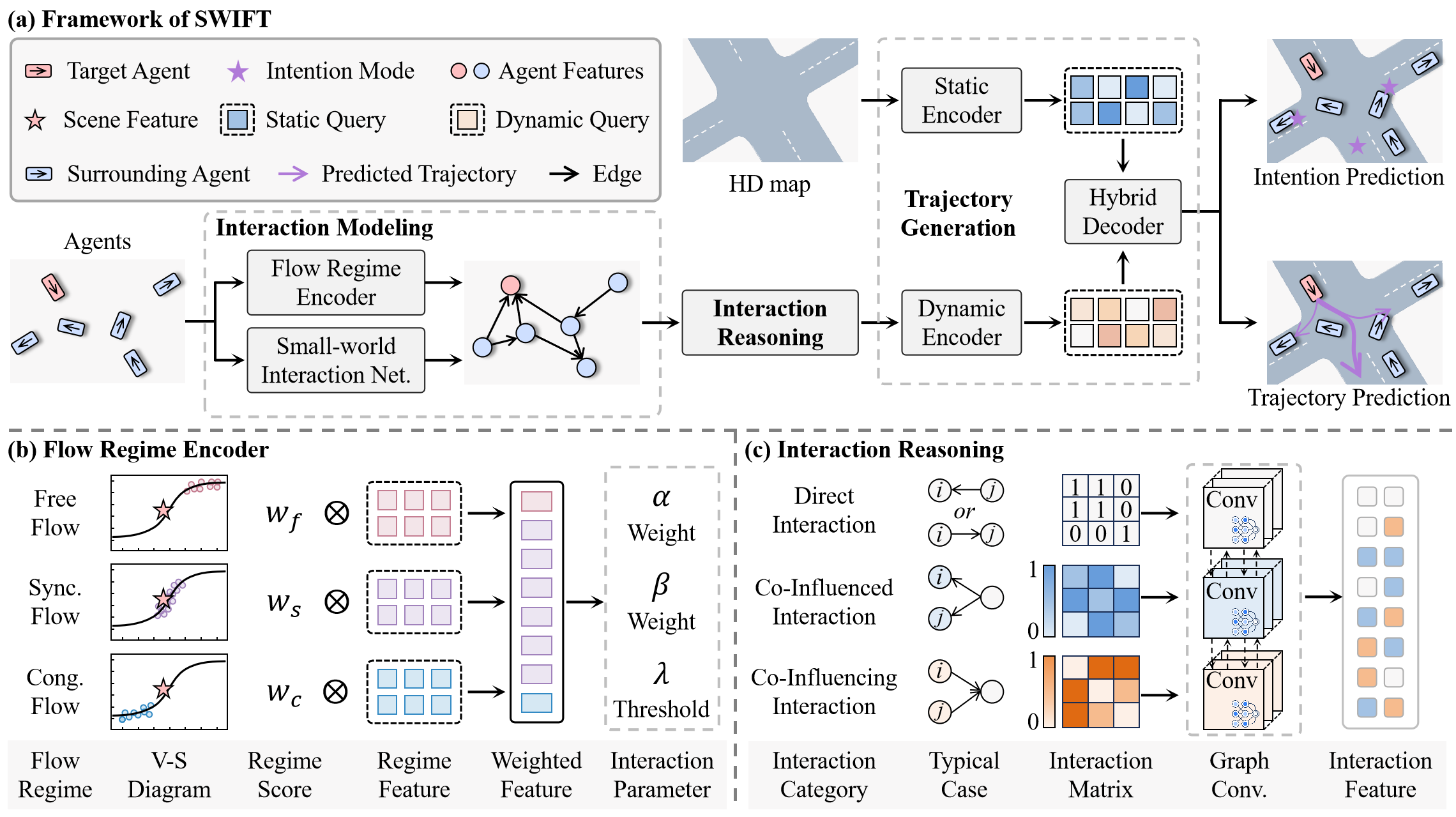}
	\caption{Overview of the proposed SWIFT model. Panel (a) presents the overall workflow, which takes agents' historical trajectories and a high-definition (HD) map as input, and predicts the target agent’s intent and future trajectories. SWIFT processes these inputs through three key stages: Interaction Modeling, Interaction Reasoning, and Trajectory Generation. Panels (b) and (c) highlight the internal designs of the Flow Regime Encoder and the Interaction Reasoning module, respectively.}
        \label{kuangjiatu}
\end{figure*}

\section{Related Work}\label{RelatedWork}
\subsection{Interaction Modeling in Traffic Scenes}
The application of deep learning to trajectory prediction in autonomous driving has evolved significantly, with early efforts treating it as a time-series forecasting problem~\cite{lan2024hi, liao2024physics}. A turning point came with Social LSTM~\cite{alahi2016social}, which introduced the concept of social forces, shifting the focus toward modeling inter-agent interactions~\cite{liao2024cognitive, deo2022multimodal, chen2024q,liao2025toward}. Building on this, grid-based methods and convolutional neural networks (CNNs) emerged to aggregate local information by discretizing scenes~\cite{chen2022intention}, yet their reliance on frame-by-frame processing neglects temporal evolution and global dependencies, compromising both accuracy and efficiency. To address these limitations, attention mechanisms like those in Nettraj~\cite{liang2021nettraj} incorporated spatiotemporal weighting to enhance interaction modeling, but their dependence on grids and quadratic computational complexity limits scalability in dense traffic scenarios. Subsequently, Graph Neural Networks (GNNs), such as SFEM-GCN~\cite{du2024social}, advanced the field by encoding agents as graph nodes with semantic and positional features to handle unstructured environments. However, these approaches primarily focus on modeling local interactions among nearby agents, while overlooking the ubiquitous yet critical global interactions. This highlights the urgent need for an interaction modeling framework to capture and integrate local and global interactions among agents efficiently.

\subsection{Small-world Network}
Small-world networks, as characterized by Watts and Strogatz~\cite{watts1998collective}, exhibit high clustering coefficients and short average path lengths, thereby facilitating efficient information dissemination through a combination of localized connectivity and global accessibility. Owing to these properties, small-world structures have been widely adopted in diverse domains, including social networks~\cite{lee2024cultural}, navigation systems~\cite{kleinberg2000navigation}, and power grids~\cite{zhao2024electric}. Ghavasieh et al.~\cite{ghavasieh2024diversity} further uncovered the underlying mechanisms of the small-world phenomenon through a comprehensive theoretical analysis of 543 real-world networks, including traffic systems. In the context of traffic networks, this phenomenon manifests when local perturbations, such as a single vehicle’s sudden braking, rapidly propagate upstream through inter-agent interactions, potentially resulting in traffic congestion~\cite{wu2011shockwave}. While small-world networks exhibit strong compatibility with interaction characteristics in traffic scenarios~\cite{wang2025nest}, the fundamental mechanisms linking small-world topology to traffic dynamics remain largely underexplored. In this work, we aim to close this gap by developing an interaction modeling framework that explicitly incorporates insights from traffic-flow theory and small-world networks, thereby providing a more theoretically consistent foundation for trajectory prediction in traffic scenarios.

\section{Methodology}\label{Method}
\subsection{Problem Formulation}
Given a traffic scenario involving $N$ traffic agents, we designate the agent of interest as the \textit{target agent}, and the remaining agents are referred to as \textit{surrounding agents}. The objective of the proposed SWIFT is to predict the target agent's future trajectory, denoted as $\hat{Y}$, over the prediction horizon $[t_{\mathrm{h}+1}:t_\mathrm{f}]$. 
The SWIFT takes as input the observations of the target agent $X^{t_{1}:t_\mathrm{h}}_{1}$ and surrounding agents $X^{t_{1}:t_\mathrm{h}}_{2:N}$ over the historical horizon $[t_{1}:t_\mathrm{h}]$, along with a high-definition (HD) map $M$. Specifically, the historical observations consist of sequential data such as agent coordinates and velocities, while the HD map encodes lane geometries.

To capture the inherent uncertainty of driving behavior, we adopt a multimodal prediction framework that has been widely recognized in the literature for modeling behavioral uncertainty~\cite{chen2022intention,liao2024bat,zhou2022hivt}. Specifically, SWIFT predicts $K$ candidate trajectories $\{\hat{Y}_k\}_{k=1}^K$, where each trajectory mode $k$ is represented as:
\begin{equation}
\hat{Y}_k^{t}
= \big(\hat{x}_k^{t}, \hat{y}_k^{t}, \mu_{k}^{\mathrm{x},t}, \mu_{k}^{\mathrm{y},t},p_k\big),
\end{equation}
where $\hat{x}_k^{t}$ and $\hat{y}_k^{t}$ denote the predicted coordinates, $\mu_{k}^{\mathrm{x},t}$ and $\mu_{k}^{\mathrm{y},t}$ are scale parameters that quantify per-axis uncertainty at time $t$, and $p_k$ is the probability of the $k$-th mode.

\subsection{Overall Framework}
Fig.~\ref{kuangjiatu} illustrates the pipeline of the proposed SWIFT, which comprises three main stages. In the interaction modeling stage, SWIFT constructs an interaction graph $G$ using a Small-world Interaction Network to represent inter-agent interactions. To adapt the interaction graph $G$ to varying traffic flow dynamics, we introduce a Flow Regime Encoder, which dynamically modulates the interaction parameters. Conditioning the constructed graph $G$, SWIFT proceeds to the interaction reasoning stage, where agent features are aggregated to infer interaction feature $\mathcal{F}_\mathrm{i}$. Finally, in the trajectory generation stage, SWIFT generates multi-modal future trajectories for the target agent.

\begin{figure}[t]
        \centering
	\includegraphics[width=0.86\linewidth]{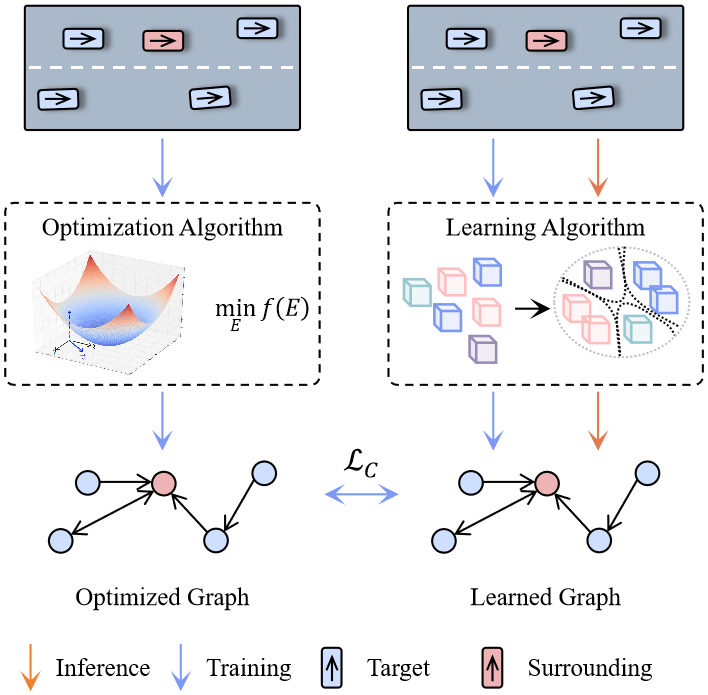}
	\caption{Illustration of the training process for Small-world Interaction Network. The optimization algorithm constructs an optimal interaction graph, which serves as the ground truth to supervise the training of the learning algorithm.}
        \label{optimal}
\end{figure}

\subsection{Small-world Interaction Network}
In contrast to existing approaches that rely primarily on spatial proximity and therefore capture only local interactions, we aim to construct an interaction graph that accounts for both local and global dependencies while supporting efficient information propagation among agents. These objectives are naturally aligned with the properties of small-world networks, which offer high local clustering and short global path lengths.
To this end, we propose a directed interaction graph $G=[V, E]$, where each vertex $v_i \in V$ represents an agent, and each directed edge $e_{ij} \in \{0,1\} \subset E$ indicates whether agent $i$ is influenced by agent $j$. The properties of the interaction graph are governed by the structural arrangement of its edge topology. Therefore, we adopt a hybrid modeling approach that focuses on the edge set $E$. Specifically, Structural Graph Optimization constructs an edge set $E_\mathrm{op}$ to explicitly enforce small-world properties, which supervises the training of Adaptive Graph Learning. This hybrid design, illustrated in Fig.~\ref{optimal}, ensures that the learned edge set $E_\mathrm{le}$ inherits small-world properties without incurring inference-time optimization overhead.

\renewcommand{\algorithmcfname}{Algorithm}
\begin{algorithm}[tbp]
\caption{Optimization-based Modeling}
\label{alg:SimulatedAnnealingSWG}
\KwIn{Target agent $v_i$, agent set $V$, distance function $\phi(\cdot,\cdot)$, local weights $\alpha$, global weights $\beta$, initial temperature $T^0$, cooling rate $\eta$, maximum iterations $H$}
\KwOut{Optimized connected agents set $\mathcal{N}^*(v_i)$}
\BlankLine
Initialize a connected agents set $\mathcal{N}(v_i) \subset V \setminus \{v_i\}$\\
Compute initial cost $\mathcal{J}(v_i)$ using Equation~\eqref{costfunction}\\
Set $T \gets T^0$\\
\For{$k \gets 1$ \textbf{to} $H$}{
    Generate a solution $\mathcal{N}'(v_i)$ by randomly adding/removing/swapping an agent in $\mathcal{N}(v_i)$\\
    Compute cost $\mathcal{J}'(v_i)$ of the new set $\mathcal{N}'(v_i)$\\
    $\Delta \gets \mathcal{J}'(v_i) - \mathcal{J}(v_i)$\\
    \uIf{$\Delta < 0$}{
        Accept new solution: $\mathcal{N}(v_i) \gets \mathcal{N}'(v_i)$; update $\mathcal{J}(v_i) \gets \mathcal{J}'(v_i)$
    }
    \Else{
        Accept with probability $p \gets \exp(-\Delta / T)$
    }
    Update temperature: $T \gets \eta \cdot T$
}
\Return Final connected agents set $\mathcal{N}^*(v_i) \gets \mathcal{N}(v_i)$
\end{algorithm}

\subsubsection{Structural Graph Optimization}
The edge optimization process operates in an agent-centric manner, where each agent $i$ independently determines its optimal connected agents set $\mathcal{N}^*(v_i)$. The optimized edge set $E_\mathrm{op}$ is subsequently constructed with following function:
\begin{equation}
e_{ij} =
\begin{cases}
1, & \text{if } v_j \in \mathcal{N}^*(v_i)  \quad \text{or} \quad i=j \\
0, & \text{otherwise}.
\end{cases}
\end{equation}

Inspired by the formation mechanisms of small-world networks~\cite{alamdari2023small}, we formulate this connected agents selection as an optimization problem minimizing the information acquisition cost $\mathcal{J}(v_i)$. For each agent $v_i$, the optimal connected agents set $\mathcal{N}^*(v_i)$ is the solution to the following optimization:
{
\begin{equation}
\label{costfunction}
\begin{aligned}
\mathcal{N}^*&=\operatorname*{argmin}_{\mathcal{N}}\mathcal{J}(v_i) \\
&= \underbrace{\alpha\sum_{v_j \in \mathcal{N}(v_i)}\phi^\varphi(v_i,v_j)}_{\text{connection cost}} + \underbrace{\beta\sum_{v_w \notin \mathcal{N}(v_i)} \min_{v_j\in \mathcal{N}(v_i)} \phi(v_w, v_j)}_{\text{isolation cost}}. 
\end{aligned}
\end{equation}

Here, $\phi(\cdot,\cdot)$ denotes the distance between agents. The first term, connection cost, penalizes excessive direct connections, enforcing local clustering efficiency. The second term, isolation cost, ensures global connectivity by penalizing the isolation of non-connected agents through their minimum distance to the connected set. The local weights $\alpha$ and global weights $\beta$ govern the trade-off between local clustering and global connectivity. The exponent $\varphi$ governs the distance penalty; its selection is discussed in Section~\ref{connectioncost}.}

To efficiently approximate the optimal connected agent set $\mathcal{N}^*(v_i)$, we employ a Simulated Annealing (SA) strategy. This method explores the discrete and non-convex search space by iteratively perturbing the currently connected agent set through localized modifications. The objective function $\mathcal{J}(v_i)$ serves as the energy function, jointly capturing the trade-off between local clustering and global connectivity. The optimization process is detailed in Algorithm \ref{alg:SimulatedAnnealingSWG}. Furthermore, we provide a theoretical analysis that establishes finite-iteration performance guarantees for the proposed optimization scheme.

\begin{theorem}[Finite-iteration Performance Guarantee of Optimization-based Modeling]
\label{thm:sa-bound}
Let $\mathcal{S}$ be the discrete solution space of all possible connected agent sets $\mathcal{N}(v_i)$ for agent $v_i$, and we run Simulated Annealing for $k$ iterations with a temperature schedule $T^k =\eta^k T^0$ with $\eta \in (0,1)$. Let $\mathcal{N}_\mathrm{SA}^{k}$ be the solution with stationary distribution at the $k$-th iteration. Then, we have the High-Probability Bound:
    \begin{equation}\label{eq:high-prob-bound}
      \mathbb{P}\left(\mathcal{J}(\mathcal{N}_\mathrm{SA}^{k}) - \mathcal{J}^* > \epsilon\right) \leq |\mathcal{S}_\epsilon|\exp\left(-\tfrac{\epsilon}{T^k}\right),
    \end{equation}
    where $\mathcal{S}_\epsilon = \{\mathcal{N} \in \mathcal{S} : \mathcal{J}(\mathcal{N}) > \mathcal{J}^* + \epsilon\}$ and $\epsilon > 0$.
\end{theorem}

\begin{proof}
For any SA Markov chain with stationary distribution $\pi_T$, 
\begin{equation}
  \pi_T(\mathcal{N}) = \frac{1}{Z_T}\exp\left(-\tfrac{\mathcal{J}(\mathcal{N})}{T}\right), \quad Z_T = \sum_{\widetilde{\mathcal{N}}} \exp\left(-\tfrac{\mathcal{J}(\widetilde{\mathcal{N}})}{T}\right),
\end{equation}
The probability of drawing a solution in $\mathcal{S}_\epsilon$ is
\begin{equation}
  \mathbb{P}_T\bigl(
    \mathcal{N}\in \mathcal{S}_\epsilon
  \bigr)
  =
  \sum_{\mathcal{N}\in \mathcal{S}_\epsilon} \pi_T(\mathcal{N})
  =
  \frac{
    \sum_{\mathcal{N}\in \mathcal{S}_\epsilon}
      \exp\Bigl(-\tfrac{\mathcal{J}(\mathcal{N})}{T}\Bigr)
  }{
    \sum_{\mathcal{N}\in \mathcal{S}}
      \exp\Bigl(-\tfrac{\mathcal{J}(\mathcal{N})}{T}\Bigr)
  }.
\end{equation}
where $  \mathcal{S}_\epsilon=\bigl\{\mathcal{N}\in \mathcal{S}\big|\mathcal{J}(\mathcal{N})>\mathcal{J}(\mathcal{N}^*) + \epsilon\bigr\}$ is the set of $\epsilon$-suboptimal solutions.

Since $\mathcal{N}\in \mathcal{S}_\epsilon$ implies
$\mathcal{J}(\mathcal{N}) \ge \mathcal{J}(\mathcal{N}^*) + \epsilon,$
we obtain
\begin{equation}
\begin{aligned}
  \exp\Bigl(-\tfrac{\mathcal{J}(\mathcal{N})}{T}\Bigr)
  &\le
  \exp\Bigl(-\tfrac{\mathcal{J}(\mathcal{N}^*) + \epsilon}{T}\Bigr)\\
  &=
  \exp\Bigl(-\tfrac{\mathcal{J}(\mathcal{N}^*)}{T}\Bigr)
  \exp\Bigl(-\tfrac{\epsilon}{T}\Bigr).
\end{aligned}
\end{equation}
Hence,
\begin{equation}
\begin{aligned}
  \mathbb{P}_T\bigl(\mathcal{N}\in \mathcal{S}_\epsilon\bigr)
  &\le
  \frac{
    \sum_{\mathcal{N}\in\mathcal{S}_\epsilon}
    \exp\Bigl(-\tfrac{\mathcal{J}(\mathcal{N}^*)}{T}\Bigr)
    \exp\Bigl(-\tfrac{\epsilon}{T}\Bigr)
  }{
    \sum_{\mathcal{N}\in\mathcal{S}}
      \exp\Bigl(-\tfrac{\mathcal{J}(\mathcal{N})}{T}\Bigr)
  }\\
  &=
  \exp\Bigl(-\tfrac{\epsilon}{T}\Bigr)
  \cdot
  \frac{
    \exp\Bigl(-\tfrac{\mathcal{J}(\mathcal{N}^*)}{T}\Bigr)
    \cdot
    \bigl|\mathcal{S}_\epsilon\bigr|
  }{
    \sum_{\mathcal{N}\in\mathcal{S}}
      \exp\Bigl(-\tfrac{\mathcal{J}(\mathcal{N})}{T}\Bigr)
  }.
\end{aligned}
\end{equation}
Moreover,
\begin{equation}
  \sum_{\mathcal{N}\in \mathcal{S}}
    \exp\Bigl(-\tfrac{\mathcal{J}(\mathcal{N})}{T}\Bigr)
  \ge
  \exp\Bigl(-\tfrac{\mathcal{J}(\mathcal{N}^*)}{T}\Bigr).
\end{equation}
Thus the ratio involving the partition function is at most $\bigl|\mathcal{S}_\epsilon\bigr|$, yielding
\begin{equation}
  \mathbb{P}_T\bigl(
    \mathcal{J}(\mathcal{N}) - \mathcal{J}(\mathcal{N}^*) > \epsilon
  \bigr)
  \le
  \bigl|\mathcal{S}_\epsilon\bigr|
  \exp\Bigl(-\tfrac{\epsilon}{T}\Bigr).
\end{equation}
Substituting $T = T^k$ yields the High-Probability Bound~\eqref{eq:high-prob-bound}.
\end{proof}

\subsubsection{Adaptive Graph Learning}\label{sec:learning_model}
To embed small-world principles within a learnable framework, we reformulate edge construction as a differentiable, parameterized objective. Specifically, the learned edge set $\hat{E}_\mathrm{le}$ is obtained via a weighted fusion of local and global interaction cues:
\begin{equation}
\hat{E}_\mathrm{le} = \phi_{\text{MLP}}\Big(\underbrace{\alpha \cdot \phi_{\mathrm{local}}(X^{t_{1}:t_\mathrm{h}}_{1:N})}_{\text{local affinity}} + \underbrace{\beta \cdot \phi_{\mathrm{global}}(X^{t_{1}:t_\mathrm{h}}_{1:N})}_{\text{global connectivity}}\Big),
\end{equation}
where $\phi_{\mathrm{local}}$ and $\phi_{\mathrm{global}}$ are learnable functions, $\alpha$ and $\beta$ are weights. Then ${E}_\mathrm{le}$ is obtained with a binarization using an interaction threshold $\gamma$.

Overall, the learning-based modeling process $f_\iota$ can be summarized as follows:
\begin{equation}
E_\mathrm{le} = f_\iota(X^{t_{1}:t_\mathrm{h}}_{1:N}, \alpha, \beta, \gamma),
\end{equation}
where local weights $\alpha$, global weights $\beta$, and interaction threshold $\gamma$ play a pivotal role in determining the fidelity of interaction modeling.

\subsection{Flow Regime Encoder}
Interaction patterns in traffic systems exhibit distinct characteristics across different flow regimes. To enable SWIFT to adapt its interaction modeling to these varying scene-level conditions, we introduce a Flow Regime Encoder that encodes flow-aware context for modulating interaction parameters, namely the local weight~$\alpha$, global weight~$\beta$, and interaction threshold~$\gamma$.
Inspired by three-phase traffic theory~\cite{kerner1999physics}, we consider three regimes: free flow, synchronized flow, and congested flow, with synchronized flow serving as the unstable intermediate regime that links the other two stable regimes. Accordingly, we analyze the stability of traffic flow to formally establish the criterion separating synchronized flow from the two stable regimes, summarized in Theorem \ref{thm:ovm-stability}.
\begin{theorem}[Stability Condition of Traffic Flow]
\label{thm:ovm-stability}
Consider the Optimal Velocity Model (OVM)~\cite{bando1995dynamical} for a one-dimensional traffic system:
\begin{equation}
    \label{eq:ovm}
    \ddot{x}_n(t) = a \left[ V\big(s_n\big) - \dot{x}_n(t) \right],
\end{equation}
where $x_n(t)$ is the position of vehicle $n$ at time $t$, $a > 0$ is the driver sensitivity parameter~\cite{bando1998analysis}, and $V(s)$ is a smooth, monotonically increasing optimal velocity function with respect to the spacing $s_n=x_{n+1} - x_n$~\cite{punzo2016speed}, with $x_{n+1}$ denotes position of the immediate preceding vehicle.

Initialize the traffic flow in a homogeneous state with constant spacing $s_0$ and velocity $v_0 = V(s_0)$. The corresponding equilibrium trajectory $x_n^{(0)}$ is
\begin{equation}
x_n^{(0)}(t) = ns_0 + V(s_0) t.
\end{equation}

This traffic flow is stable if and only if 
\begin{equation}
V'(s_0) < \frac{a}{2},
\end{equation}
where \( V'(\cdot) \) is the derivative of the function \( V(\cdot) \).
\end{theorem}

\begin{proof}
Model the observed trajectory of vehicle $n$
\begin{equation}
x_n(t)=x_n^{(0)}(t)+{y}_n(t), \qquad |{y}_n(t)|\ll 1,
\end{equation}
where ${y}_n(t)$ denotes small perturbations due to driving actions such as braking or acceleration.

Linearizing the dynamics about the homogeneous equilibrium state $(s_0,v_0)=(s_0,V(s_0))$ yields
\begin{equation}
\label{eq:linearized-ovm}
\ddot{y}_n(t) =- a\dot{y}_n(t) + a V'(s_0) \big( y_{n+1}(t) - y_n(t) \big)
\end{equation}

Since the linearized system is spatially translation-invariant, its dynamics can be diagonalized by the discrete Fourier transform: 
\begin{equation} 
\label{eq:fourier} 
\hat{y}(k,t) := \sum_{n\in\mathbb{Z}} y_n(t) e^{-i k n}, \qquad k \in \mathbb{T} := [-\pi,\pi]. 
\end{equation} 

By the Inverse Fourier Transform~\cite{heckbert1995fourier}, the perturbation can be reconstructed as
\begin{equation}
\label{eq:inv-fourier}
y_n(t) = \frac{1}{2\pi}\int_{-\pi}^{\pi} \hat{y}(k,t) e^{i k n}  dk.
\end{equation}
which shows that an arbitrary perturbation is a superposition of sinusoidal (Fourier) modes with different wavenumbers $k$. 

The sinusoidal mode with wavenumber $k$ thus takes the form
\begin{equation} 
\label{eq:sin-mode} 
y_n(t) = \hat{y}(k,t) e^{i k n}. 
\end{equation} 

Substituting \eqref{eq:sin-mode} into \eqref{eq:linearized-ovm} gives the modal ordinary differential equation
\begin{equation} 
\label{eq:modal-ode} \ddot{\hat{y}}(k,t) + a\dot{\hat{y}}(k,t) - a V'(s_0) \big( e^{i k} - 1 \big) \hat{y}(k,t) = 0, 
\end{equation}
whose characteristic equation is 
\begin{equation} 
\label{eq:char} 
\lambda^2 + a \lambda - a V'(s_0)\big( e^{i k} - 1 \big) = 0, 
\end{equation}
with two roots
\begin{equation}\label{eq:R}
\lambda_\pm(k)= -\frac{a}{2} \pm \sqrt{\frac{a^2}{4}+aV'(s_0)(e^{ik}-1)}.
\end{equation}

Let
\begin{equation}
w(k):=\frac{a^2}{4}+aV'(s_0)(e^{ik}-1)=u+iv,
\end{equation}
where
\begin{equation}
u=\frac{a^2}{4}-aV'(s_0)(1-\cos k),\quad v=aV'(s_0)\sin k.
\end{equation}

The real part of the square root of a complex number has the closed form
\begin{equation}
\Re\sqrt{u+iv}=\sqrt{\frac{|w|+u}{2}},\qquad |w|=\sqrt{u^2+v^2}.
\end{equation}

Hence, the real part of the slow root is
\begin{equation}\label{eq:star}
\Re\lambda_+(k)=-\frac{a}{2}+\sqrt{\frac{|w|+u}{2}},
\end{equation}
and the real part of the fast root is
\begin{equation}\label{eq:Re-lambda-minus}
\Re\lambda_-(k)
= -\frac{a}{2} - \sqrt{\frac{|w| + u}{2}}<0.
\end{equation}

According to Lemma \ref{lem:spectral-stability}, the criterion for stability of the traffic flow is:
\begin{equation}
\label{criterion_1}
\Re\lambda_\pm(k)<0.
\end{equation}

To investigate when the stability condition \eqref{criterion_1} holds, we consider three cases: \(V'(s_0) < \frac{a}{2}\), \(V'(s_0) > \frac{a}{2}\), and \(V'(s_0) = \frac{a}{2}\). Moreover, since $\Re\lambda_-(k)<-\frac{a}{2}<0$ is always negative, it suffices to prove \(\Re\lambda_+(k) < 0\).

From \eqref{eq:star}, the condition $\Re\lambda_+(k)<0$ is equivalent to
\begin{equation}
\sqrt{\frac{|w| + u}{2}} < \frac{a}{2}
\quad\Longleftrightarrow\quad
|w| < \frac{a^2}{2} - u.
\end{equation}

Squaring both sides yields
\begin{equation}
u^2 + v^2 < 
\left(\frac{a^2}{4} + aV'(s_0)(1 - \cos k)\right)^2.
\end{equation}

Expanding and simplifying both sides leads to the inequality
\begin{equation}\label{eq:key-ineq}
a(1-\cos k) > V'(s_0)\sin^2 k.
\end{equation}

Using the trigonometric identity
\begin{equation}
\sin^2 k = 2(1-\cos k) - (1-\cos k)^2 \le 2(1-\cos k),
\end{equation}
we obtain
\begin{equation}
V'(s_0)\sin^2 k \le 2V'(s_0)(1-\cos k).
\end{equation}

Hence, inequality \eqref{eq:key-ineq} holds whenever
\begin{equation}
2V'(s_0) < a
\quad\Longleftrightarrow\quad
V'(s_0) < \frac{a}{2}.
\end{equation}

Thus, in the case $V'(s_0) < a/2$, the condition $\Re\lambda(k)<0$ holds, implying that the traffic flow is stable.

In traffic system stability analysis, short-wave perturbations are always attenuated~\cite{wilson2011car}. Therefore, we only need to analyze the slow root $\lambda_+(k)$ in the long-wave perturbations ($|k|\ll 1$), expanding \eqref{eq:R} gives
\begin{equation}
\Re\lambda_+(k)
= V'(s_0)\!\left(\frac{V'(s_0)}{a}-\frac12\right) k^2 + O(k^4).
\end{equation}

Hence, in the case $V'(s_0) > a/2$, the coefficient of \(k^2\) is positive, implying that $\Re\lambda_+(k)>0$. Therefore, a growing long-wave mode exists, and the traffic flow is linearly unstable.

Based on the above analysis of the two cases, the case \(V'(s_0)=\tfrac{a}{2}\) corresponds to the marginal state. As stated in~\cite{wilson2008mechanisms}, traffic systems cannot sustain stability at marginal conditions, and we therefore classify this case as unstable.

Combining the above three cases, we conclude that the traffic flow is stable if and only if
\begin{equation}
\boxed{
V'(s_0) < \frac{a}{2}
}.
\end{equation}
\end{proof}

Grounded in Theorem~\ref{thm:ovm-stability}, synchronized flow corresponds to the unstable scenario $V'(s_0) \geq \frac{a}{2}$. 
\begin{definition}[Synchronized Flow Regime]
\label{def:transitional}
The Synchronized Flow Regime refers to scenarios with spacing $s_0=1/k$, where $k$ denotes traffic density~\cite{helbing1998coherent}, that satisfy
\begin{equation}
V'(s_0) \geq \frac{a}{2}.
\end{equation}

In this regime, drivers respond sensitively to fluctuations in spacing, and small perturbations can grow over time. This regime is linearly unstable and is associated with the emergence of stop-and-go traffic waves. It represents the phase where vehicle interactions are most pronounced and dynamically coupled over long distances.
\end{definition}

On the stable side \(V'(s_0) < a/2\), the traffic flow separates into two regimes, namely free flow and congested flow, as defined below.

\begin{definition}[Free Flow Regime]
\label{def:free-flow}
The Free Flow Regime refers to scenarios that satisfy
\begin{equation}
V'(s_0) < \frac{a}{2}, \quad \text{with } V'(s_0) \approx 0.
\end{equation}

In this regime, vehicles are sparsely distributed, and interactions between them are negligible. Each vehicle typically travels at or near the free-flow speed $v_\mathrm{f}$.
\end{definition}

\begin{definition}[Congested Flow Regime]
\label{def:congested}
The Congested Flow Regime refers to scenarios that satisfy
\begin{equation}
V'(s_0) < \frac{a}{2}, \quad \text{with } V(s_0) \approx 0.
\end{equation}

Here, vehicles are tightly packed and strongly coupled at a local level. However, due to low speeds and minimal spacing, drivers exhibit reduced responsiveness, leading to the re-establishment of local linear stability. Perturbations in this regime tend to be absorbed locally and do not propagate widely through the system.
\end{definition}

These definitions delineate three regimes with distinct interaction patterns, motivating a regime-aware encoding. Concretely, the Flow Regime Encoder begins by extracting a scene feature $\mathcal{F}_{\mathrm{s}}$, which encodes key scene-level statistics such as average velocity and spacing. This feature is passed through a set of regime-specific encoders $\phi_{r}(\cdot)$, where $r \in \{\mathrm{f}, \mathrm{s}, \mathrm{c}\}$ corresponds to the free flow, synchronized flow, and congested flow regimes, respectively. Each encoder produces a latent feature $h_r$ along with a regime confidence score $w_{r}$:
\begin{equation}
(h_r,\ w_r) = \phi_r(\mathcal{F}_{\mathrm{s}}), \quad \text{for } r \in \{\mathrm{f}, \mathrm{s}, \mathrm{c}\}.
\label{eq:fre-encoders}
\end{equation}

Finally, these regime-specific features are aggregated via a weighted summation and subsequently passed through an MLP-based decoder to generate the local weights $\alpha$, global weights $\beta$, and interaction threshold $\gamma$. This formulation enables the model to dynamically adjust its interaction modeling in response to various traffic regimes, thereby enhancing its generalization across diverse scenarios.

\begin{lemma}[Stability Criterion]
\label{lem:spectral-stability}
Consider the linearized dynamics of the perturbation $y_n(t)$ around the
homogeneous equilibrium $(s_0,v_0)=(s_0,V(s_0))$:
\begin{equation}
\label{eq:linearized-ovm-lemma}
\ddot{y}_n(t)
= -a\dot{y}_n(t) + aV'(s_0)\big(y_{n+1}(t) - y_n(t)\big),
\quad n\in\mathbb{Z}.
\end{equation}

The traffic flow is stable in $\ell^2$
if there exist constants $M,\gamma>0$ such that, for every initial data
$(y(0),\dot y(0))\in\ell^2(\mathbb Z)\times\ell^2(\mathbb Z)$, the corresponding solution satisfies
\begin{equation}
\label{eq:exp-stability-def}
\|y(t)\|_{\ell^2}
\le M e^{-\gamma t}\Big(\|y(0)\|_{\ell^2} + \|\dot y(0)\|_{\ell^2}\Big).
\end{equation}

Then the traffic flow is stable in the sense of \eqref{eq:exp-stability-def} if and only if
\begin{equation}
\label{criterion}
\Re\lambda_\pm(k)<0
\quad\text{for all } k\in[-\pi,\pi].
\end{equation}
\end{lemma}

\begin{proof}
By the definition of the discrete Fourier transform \eqref{eq:fourier} and
its inverse \eqref{eq:inv-fourier}, any perturbation
$\{y_n(t)\}_{n\in\mathbb{Z}}$ with $y(t)\in\ell^2(\mathbb Z)$ can be
decomposed as
\begin{equation}
y_n(t) = \frac{1}{2\pi}\int_{-\pi}^{\pi} \hat{y}(k,t) e^{ikn} dk,
\end{equation}
and Parseval's identity~\cite{kelkar2007extension} yields
\begin{equation}
\label{eq:parseval}
\|y(t)\|_{\ell^2}^2 := \sum_{n\in\mathbb{Z}} |y_n(t)|^2
= \frac{1}{2\pi}\int_{-\pi}^{\pi} |\hat{y}(k,t)|^2 dk.
\end{equation}

Substituting the Fourier ansatz
\begin{equation}
y_n(t) = \hat{y}(k,t)e^{ikn}
\end{equation}
into \eqref{eq:linearized-ovm-lemma} yields the decoupled modal ordinary differential equation
\begin{equation}
\label{eq:modal-ode-lemma}
\ddot{\hat{y}}(k,t) + a\dot{\hat{y}}(k,t)
- aV'(s_0)(e^{ik}-1)\hat{y}(k,t) = 0,
\end{equation}
or, in first-order form,
\begin{equation}
\partial_t
\begin{pmatrix}
\hat y(k,t)\\
\partial_t\hat y(k,t)
\end{pmatrix}
= A(k)
\begin{pmatrix}
\hat y(k,t)\\
\partial_t\hat y(k,t)
\end{pmatrix},
\end{equation}
\begin{equation}
A(k)=
\begin{pmatrix}
0 & 1\\
aV'(s_0)(e^{ik}-1) & -a
\end{pmatrix}.
\end{equation}

The characteristic equation of \eqref{eq:modal-ode-lemma} is
\begin{equation}
\lambda^2 + a\lambda - aV'(s_0)(e^{ik}-1)=0,
\end{equation}
with roots $\lambda_\pm(k)$. These depend continuously on $k$, and the same
is true for the matrices $A(k)$.

For each fixed $k\in[-\pi,\pi]$ and initial data
\begin{equation}
\mathbf u_0(k)
:=
\begin{pmatrix}
\hat y(k,0)\\
\partial_t\hat y(k,0)
\end{pmatrix}
\in\mathbb C^2,
\end{equation}
the solution of the modal system can be written as
\begin{equation}
\label{eq:modal-semigroup}
\begin{pmatrix}
\hat y(k,t)\\
\partial_t\hat y(k,t)
\end{pmatrix}
= e^{A(k)t}\mathbf u_0(k),
\qquad t\ge0,
\end{equation}
where $e^{A(k)t}$ denotes the $2\times2$ matrix exponential. We denote $\mu(k):=\max\{\Re\lambda_+(k),\Re\lambda_-(k)\}$.

\medskip
\noindent\emph{Sufficiency.}
Assume \eqref{criterion} holds, i.e.,
$\Re\lambda_\pm(k)<0$ for all $k\in[-\pi,\pi]$. Then
$\mu(k)<0$ and, by continuity of $\lambda_\pm(k)$ in $k$ and compactness
of $[-\pi,\pi]$, there exists $\gamma>0$ such that
\begin{equation}
\mu(k)\le -\gamma
\quad\text{for all } k\in[-\pi,\pi].
\end{equation}

Fix $k\in[-\pi,\pi]$. Since $A(k)$ is a $2\times2$ complex matrix whose
spectrum lies in $\{z\in\mathbb C:\Re z\le -\gamma\}$, standard finite dimensional semigroup theory implies that
\begin{equation}
\|e^{A(k)t}\| \le C(k) e^{-\gamma t},
\qquad t\ge0,
\end{equation}
for some finite constant $C(k)\ge1$ depending continuously on $k$ (here
$\|\cdot\|$ denotes any matrix norm on $\mathbb C^{2\times2}$). The map
$k\mapsto C(k)$ is continuous and $[-\pi,\pi]$ is compact, so we can choose
a uniform constant
\begin{equation}
C_* := \max_{k\in[-\pi,\pi]} C(k) <\infty
\end{equation}
such that
\begin{equation}
\label{eq:uniform-semigroup-bound}
\|e^{A(k)t}\| \le C_* e^{-\gamma t}
\quad\text{for all } t\ge0,\ k\in[-\pi,\pi].
\end{equation}

From \eqref{eq:modal-semigroup} and \eqref{eq:uniform-semigroup-bound} we obtain
\begin{equation}
|\hat y(k,t)|
\le \left\|
\begin{pmatrix}
\hat y(k,t)\\
\partial_t\hat y(k,t)
\end{pmatrix}
\right\|
\le C_* e^{-\gamma t}
\left\|
\begin{pmatrix}
\hat y(k,0)\\
\partial_t\hat y(k,0)
\end{pmatrix}
\right\|,
\end{equation}
and hence, for some constant $C_1>0$,
\begin{equation}
|\hat y(k,t)|
\le C_1 e^{-\gamma t}\Big(|\hat y(k,0)| + |\partial_t\hat y(k,0)|\Big),
\quad
\ t\ge0.
\end{equation}

Squaring, integrating over $k$, and using the inequality
$(a+b)^2\le 2(a^2+b^2)$, we obtain
\begin{equation}
\begin{aligned}
\|y(t)\|_{\ell^2}^2
&= \frac{1}{2\pi}\int_{-\pi}^{\pi}|\hat y(k,t)|^2dk\\
&\le \frac{C_1^2 e^{-2\gamma t}}{2\pi}
\int_{-\pi}^{\pi}
\Big(|\hat y(k,0)| + |\partial_t\hat y(k,0)|\Big)^2dk\\
&\le \frac{2C_1^2 e^{-2\gamma t}}{2\pi}
\int_{-\pi}^{\pi}
\Big(|\hat y(k,0)|^2 + |\partial_t\hat y(k,0)|^2\Big)dk\\
&= C_2^2 e^{-2\gamma t}
\Big(\|y(0)\|_{\ell^2}^2 + \|\dot y(0)\|_{\ell^2}^2\Big)
\end{aligned}
\end{equation}
for some constant $C_2>0$. Taking square roots and renaming constants, we
obtain \eqref{eq:exp-stability-def}. This proves stability.

\medskip
\noindent\emph{Necessity.}
We now show that \eqref{eq:exp-stability-def} implies
\eqref{criterion}. Suppose, for contradiction, that
\eqref{criterion} fails. Then there exists $k_0\in[-\pi,\pi]$ such that
$\mu(k_0)\ge0$. We treat the two cases
$\mu(k_0)>0$ and $\mu(k_0)=0$ simultaneously.

Let $\gamma>0$ be such that \eqref{eq:exp-stability-def} holds. Set
\begin{equation}
\varepsilon := \frac{\gamma}{2}>0.
\end{equation}

By continuity of $\mu(k)$ in $k$, there exists $\delta>0$ and a
neighbourhood
\begin{equation}
U := (k_0-\delta,k_0+\delta)\cap[-\pi,\pi]
\end{equation}
such that
\begin{equation}
\mu(k) \ge \mu(k_0) - \varepsilon
\quad\text{for all } k\in U.
\end{equation}

In particular, if $\mu(k_0)\ge0$ then
\begin{equation}
\mu(k) \ge -\varepsilon
\quad\text{for all } k\in U.
\end{equation}

For each $k\in U$ choose an eigenvalue $\lambda_j(k)$, $j\in\{+,-\}$, such
that $\Re\lambda_j(k) = \mu(k)$, and consider the corresponding eigenvector
$v_j(k)$ of $A(k)$. For the matrix
\begin{equation}
A(k)=
\begin{pmatrix}
0 & 1\\
aV'(s_0)(e^{ik}-1) & -a
\end{pmatrix},
\end{equation}
an eigenvector associated with $\lambda_j(k)$ can be taken as
\begin{equation}
v_j(k) =
\begin{pmatrix}
1\\
\lambda_j(k)
\end{pmatrix},
\end{equation}
whose first component is identically $1$ and never vanishes.

Choose a nontrivial function $\psi\in C_c^\infty(U)$ and define the initial
modal data
\begin{equation}
\label{eq:bad-initial-data}
\begin{pmatrix}
\hat y(k,0)\\
\partial_t\hat y(k,0)
\end{pmatrix}
=
\psi(k)
\begin{pmatrix}
1\\
\lambda_j(k)
\end{pmatrix},
\qquad k\in[-\pi,\pi].
\end{equation}

Then $\hat y(\cdot,0),\partial_t\hat y(\cdot,0)\in L^2(-\pi,\pi)$, so by Parseval's identity the corresponding initial data
$(y(0),\dot y(0))$ belong to $\ell^2(\mathbb Z)\times\ell^2(\mathbb Z)$.

For $k\in U$, the solution of the modal system started from
\eqref{eq:bad-initial-data} is explicitly
\begin{equation}
\begin{pmatrix}
\hat y(k,t)\\
\partial_t\hat y(k,t)
\end{pmatrix}
= e^{\lambda_j(k)t}\psi(k)
\begin{pmatrix}
1\\
\lambda_j(k)
\end{pmatrix},
\end{equation}
so in particular
\begin{equation}
|\hat y(k,t)| = |\psi(k)|e^{\Re\lambda_j(k)t}
\ge |\psi(k)|e^{-\varepsilon t},
\quad k\in U.
\end{equation}

Using \eqref{eq:parseval} and the fact that $\psi$ is supported in $U$, we
obtain the lower bound
\begin{equation}
\label{eq:lower-bound-yt}
\begin{aligned}
\|y(t)\|_{\ell^2}^2
&= \frac{1}{2\pi}\int_{-\pi}^{\pi}|\hat y(k,t)|^2dk\\
&\ge \frac{1}{2\pi}\int_U |\psi(k)|^2 e^{-2\varepsilon t}dk
= C_3 e^{-2\varepsilon t},
\end{aligned}
\end{equation}
for some constant $C_3>0$ depending only on $\psi$ and $U$.

On the other hand, from \eqref{eq:bad-initial-data} we have
\begin{equation}
|\hat y(k,0)|^2 + |\partial_t\hat y(k,0)|^2
= |\psi(k)|^2\big(1 + |\lambda_j(k)|^2\big),
\end{equation}
so
\begin{equation}
\begin{aligned}
\|y(0)\|_{\ell^2}^2 + \|\dot y(0)\|_{\ell^2}^2
=& \frac{1}{2\pi}\int_{-\pi}^{\pi}
\Big(|\hat y(k,0)|^2 \\&+ |\partial_t\hat y(k,0)|^2\Big)dk\le C_4
\end{aligned}
\end{equation}
for some constant $C_4>0$ independent of $t$.

If the exponential estimate \eqref{eq:exp-stability-def} held for all
initial data, it would in particular hold for the above choice, and we have
\begin{equation}
\begin{aligned}
\|y(t)\|_{\ell^2}^2
&\le M^2 e^{-2\gamma t}
\Big(\|y(0)\|_{\ell^2}^2 + \|\dot y(0)\|_{\ell^2}^2\Big)\\
&\le M^2 C_4 e^{-2\gamma t}.
\end{aligned}
\end{equation}

Combining this with the lower bound \eqref{eq:lower-bound-yt} yields
\begin{equation}
C_3 e^{-2\varepsilon t}
\le \|y(t)\|_{\ell^2}^2
\le M^2 C_4 e^{-2\gamma t},
\end{equation}
for all $t\ge0$. Since $\varepsilon=\gamma/2$, this implies
\begin{equation}
C_3 \le M^2 C_4 e^{-(2\gamma-2\varepsilon)t}
= M^2 C_4 e^{-\gamma t}\to0\quad\text{as }t\to\infty,
\end{equation}
which is impossible. This contradiction shows that our assumption that
$\mu(k_0)\ge0$ must be false. Hence $\mu(k)<0$ for all $k$, i.e.,
$\Re\lambda_\pm(k)<0$ for all $k\in[-\pi,\pi]$.

Combining the sufficiency and necessity arguments completes the proof.
\end{proof}

\subsection{Interaction Reasoning}
The interaction graph $G = [V, E_{\mathrm{le}}]$ explicitly encodes both the interaction relationships among agents. To efficiently aggregate information from interacting agents and derive interaction features $\mathcal{F}_\mathrm{i}$, we introduce a multi-channel graph representation~\cite{tong2020directed}, where each channel captures a distinct semantic aspect of the interactions. Specifically, we design three types of adjacency matrices:
\begin{definition}[Direct Interaction Matrix] \label{af}
To promote effective message passing between agents with explicit direct interactions, we define the Direct Interaction Matrix $M_{\mathrm{p}}$ as:
\begin{equation}
M_{\mathrm{p}}(i,j) = 
\begin{cases}
1 & \text{if } e_{ij} = 1 \text{ or } e_{ji} = 1 \\
0 & \text{otherwise.}
\end{cases}
\end{equation}

Here, $M_{\mathrm{p}}$ is a matrix indicating whether a direct interaction exists between agent $i$ and agent $j$. $e_{ij} \in \{0,1\} \subset E_{\mathrm{le}}$
\end{definition}

\begin{definition}[Co-Influenced Matrix] \label{ain}
To capture second-order interactions between agents influenced by the same agent, we define the Co-Influenced Matrix $M_{\mathrm{i}}$ as:
\begin{equation}
M_{\mathrm{i}}(i,j) = \sum_{k} \frac{e_{ik} \cdot e_{jk}}{\sum_{p} e_{pk}}.
\end{equation}
\end{definition}

\begin{definition}[Co-Influencing Matrix] \label{aout}
To represent agents that are jointly influencing the same agent, we introduce the Co-Influencing Matrix $M_{\mathrm{t}}$ as:
\begin{equation}
M_{\mathrm{t}}(i,j) = \sum_{k} \frac{e_{ki} \cdot e_{kj}}{\sum_{p} e_{kp}}.
\end{equation}
\end{definition}

The three interaction matrices drive our directed graph convolution through distinct propagation paths:
\begin{equation}
H_\mathrm{p}^{(l)}  = D_\mathrm{p}^{-\frac{1}{2}} M_{\mathrm{p}} D_\mathrm{p}^{-\frac{1}{2}} V^{(l)} W_\mathrm{p}^{(l)},    
\end{equation}
\begin{equation}
H_\mathrm{i}^{(l)} =  D_\mathrm{i}^{-\frac{1}{2}} M_{\mathrm{i}} D_\mathrm{i}^{-\frac{1}{2}} V^{(l)} W_\mathrm{i}^{(l)},  
\end{equation}
\begin{equation}
H_\mathrm{t}^{(l)} =  D_\mathrm{t}^{-\frac{1}{2}} M_{\mathrm{t}} D_\mathrm{t}^{-\frac{1}{2}} V^{(l)} W_\mathrm{t}^{(l)},   
\end{equation}
\begin{equation}
V^{(l+1)} =  f_{\Theta}(H_\mathrm{p}^{(l)},H_\mathrm{i}^{(l)},H_\mathrm{t}^{(l)}; \Theta),   
\end{equation}
where \( D_{x} = \mathrm{diag}\left(\sum_{j=1}^{n} M_x(i,j)\right) \) for \( x \in \{\mathrm{p}, \mathrm{i}, \mathrm{t}\} \), and \( V^{(0)} = \{ v_i \}_{i=1}^N \) denotes the initial agent embeddings derived from their historical trajectories \( X_{1:N} \). The parameters \( \Theta, W_\mathrm{p}^{(l)}, W_\mathrm{i}^{(l)}, W_\mathrm{t}^{(l)} \) are learnable weight matrices.

After multiple iterations of interaction-aware message propagation, we obtain the final interaction feature $\mathcal{F}_\mathrm{i}$ via a decoding layer:
\begin{equation}
\mathcal{F}_\mathrm{i} = \phi_{\mathrm{MLP}}(H_\mathrm{p}, H_\mathrm{i}, H_\mathrm{t}).
\end{equation}

\subsection{Trajectory Generation}
Given the interaction features $\mathcal{F}_\mathrm{i}$ capturing agent dynamics, our objective is to decode multi-modal future trajectories. We propose a framework addressing two critical factors: dynamic agent interactions and static road topology. The framework employs $K$ parallel trajectory generators (corresponding to distinct driving intents) with three core components:
\subsubsection{Static Encoder}
The road topology imposes constraints on feasible maneuvers by shaping the agent’s spatial decision space. To incorporate these critical static semantics into the prediction process, we begin by initializing the static query $\mathcal{Q}_\mathrm{s}$ with historical trajectories. Meanwhile, a graph-based encoder extracts the road feature $\mathcal{F}_\mathrm{r}$ from the HD map. These components are then fused through an attention module $\phi_{\mathrm{A}}$:
\begin{equation}
\mathcal{Q}_\mathrm{s} = \phi_{\mathrm{A}}(\mathcal{Q'} = \mathcal{Q}_\mathrm{s}, \mathcal{K'} = \mathcal{F}_\mathrm{r} + PE(\mathcal{F}_\mathrm{r}), \mathcal{V'} = \mathcal{F}_\mathrm{r}),
\end{equation}
where $PE(\cdot)$ is a positional encoding function.

To further enhance contextual reasoning along both temporal directions, the output is subsequently refined by a bi-directional Mamba block $\phi_{\mathrm{BM}}$~\cite{zhang2024decoupling}:
\begin{equation}
\mathcal{Q}_\mathrm{s} = \phi_{\mathrm{BM}}(\mathcal{Q}_\mathrm{s}).
\end{equation}
\subsubsection{Dynamic Encoder}
Inter-agent dynamics govern the evolution of traffic scenes by creating time-varying interaction patterns. To capture these essential motion dependencies, we first initialize the dynamic query $\mathcal{Q}_\mathrm{d}$ with interaction features. An attention module $\phi_{\mathrm{A}}$ is then applied to capture the relational structure among interacting agents:
\begin{equation}
\mathcal{Q}_\mathrm{d} = \phi_{\mathrm{A}}(\mathcal{Q'} = \mathcal{Q}_\mathrm{d}, \mathcal{K'} = \mathcal{F}_\mathrm{i} + PE(\mathcal{F}_\mathrm{i}), \mathcal{V'} = \mathcal{F}_\mathrm{i}).
\end{equation}

Subsequently, to reinforce temporal consistency and refine the motion context over time, we apply an additional attention module over the updated query:
\begin{equation}
\mathcal{Q}_\mathrm{d} = \phi_{\mathrm{A}}(\mathcal{Q'} = \mathcal{Q}_\mathrm{d}, \mathcal{K'} = \mathcal{Q}_\mathrm{d} + PE(\mathcal{Q}_\mathrm{d}), \mathcal{V'} = \mathcal{Q}_\mathrm{d}).
\end{equation}

\subsubsection{Hybrid Decoder}
The decoder synthesizes static topological constraints and dynamic interaction patterns through a multi-stage fusion process. We first generate hybrid query $\mathcal{Q}_\mathrm{h}$ via attention module:
\begin{equation}
\label{hybrid1}
\mathcal{Q}_\mathrm{h} = \phi_{\mathrm{A}}\left(
\mathcal{Q'} = \mathcal{Q}_\mathrm{d}, 
\mathcal{K'} = \mathcal{Q}_\mathrm{s}, 
\mathcal{V'} = \mathcal{Q}_\mathrm{s}
\right).
\end{equation}

To further refine this hybrid representation and ensure spatio-temporal coherence, we perform:
\begin{equation}
\label{hybrid2}
\mathcal{Q}_\mathrm{h} = \phi_{\mathrm{A}}\left(
\mathcal{Q'} = \mathcal{Q}_\mathrm{h}, 
\mathcal{K'} = \mathcal{Q}_\mathrm{h}, 
\mathcal{V'} = \mathcal{Q}_\mathrm{h}
\right).
\end{equation}

Then, we aggregate all contextual signals into a unified mode feature $\mathcal{F}_\mathrm{m}$ using a bi-directional Mamba block:
\begin{equation}
\label{hybrid3}
\mathcal{F}_\mathrm{m} = \phi_{\mathrm{BM}}\left(\mathcal{Q}_\mathrm{h} \oplus \mathcal{Q}_\mathrm{s} \oplus \mathcal{Q}_\mathrm{d}\right).
\end{equation}

This feature encodes motion semantics across modalities and temporal scales. The final trajectory prediction $\hat{Y}$ is obtained by passing $\mathcal{F}_\mathrm{m}$ through an MLP decoder.

\section{Experiment}\label{Experiment}
\subsection{Experimental Setups}

\subsubsection{Datasets}
To evaluate the effectiveness of SWIFT across diverse traffic scenarios, we conduct experiments on three real-world datasets: nuScenes~\cite{caesar2020nuscenes}, MoCAD~\cite{liao2024physics}, and NGSIM~\cite{us2008ngsim}, which collectively cover a broad range of driving environments. From the perspective of traffic regimes, these datasets exhibit substantially different and highly imbalanced distributions. In NGSIM, free flow accounts for 90.48\% of all samples, while synchronized and congested flow constitute only 3.98\% and 5.54\%, respectively. In nuScenes, synchronized flow is the prevailing regime (72.55\%), while free flow and congested flow account for 10.69\% and 16.76\%. MoCAD is dominated by congested traffic: 63.15\% of all scenes fall into the congested flow, while 31.27\% correspond to synchronized flow, and only 5.58\% represent free flow.

\subsubsection{Metrics}
As different datasets adopt different evaluation protocols, we follow the conventional accuracy metrics commonly used in each dataset to ensure a fair comparison with state-of-the-art baselines. Specifically, we employ Minimum Average Displacement Error (minADE) and Minimum Final Displacement Error (minFDE)~\cite{wang2025wake} on the nuScenes dataset, while Root Mean Square Error (RMSE)~\cite{zhang2024decoupling} is used for evaluation on the MoCAD and NGSIM datasets.

\begin{table}[t]
\centering
\caption{Performance comparison of various models on the nuScenes dataset. \textbf{Bold} values represent the best performance in each category. ``-'' denotes the missing value.}
\resizebox{\linewidth}{!}
{
\begin{tabular}{ccccccc}
\bottomrule
\rowcolor{gray!15}\text{Model} & Venue & $\text{minADE}_{5}$ & $\text{minADE}_1$ & $\text{minFDE}_1$ \\
\hline
Trajectron++~\cite{salzmann2020trajectron++} &     ECCV'20 & 1.88 & - & 9.52 \\
MultiPath~\cite{chai2020multipath} &   CoRL'20 & 1.44 &  {3.16} & 7.69 \\
PGP~\cite{deo2022multimodal} &   CoRL'21 &  1.30 & - & -  \\
AgentFormer~\cite{yuan2021agentformer} &   ICCV'21 & 1.97 & - & - \\
LaPred~\cite{kim2021lapred} &   CVPR'21 & 1.53 & 3.51 & 8.12 \\
STGM~\cite{zhong2022stgm} &  TITS'22 & - & 3.21 & 9.62\\ 
GoHome~\cite{gilles2022gohome} &   ICRA'22 &  {1.42} & - &  {6.99} \\
ContextVAE~\cite{xu2023context} &   RAL'23 &  1.59 & 3.54 & 8.24 \\
EMSIN~\cite{ren2024emsin} &    TFS'24 & 1.77 & 3.56 & - \\
CoT-Drive~\cite{liao2025cot} & TAI'25 & 1.56& -&- \\
SeFlow~\cite{zhang2024seflow} & ECCV'24  & 1.38 & - & 7.89 \\
WAKE~\cite{wang2025wake} & TPAMI'25  & 1.24 & 3.03 & 7.02\\
 \hline
\textbf{SWIFT} & - & \textbf{1.15} & \textbf{2.89} & \textbf{6.73}\\
\toprule
\end{tabular}}
\label{table:performance_nuscene}
\end{table}

\begin{table}[th]
  \centering
    \caption{Experimental results on MoCAD. Metric: RMSE}
  \setlength{\tabcolsep}{3.5mm}
  \resizebox{0.9\linewidth}{!}{
    \begin{tabular}{cccccc}
    \bottomrule
\rowcolor{gray!15}{Model} & {1\,s} & {2\,s} & {3\,s} & {4\,s} & {5\,s} \\
    \hline
    S-LSTM~\cite{alahi2016social} & 1.73  & 2.46  & 3.39  & 4.01  & 4.93 \\
    S-GAN~\cite{gupta2018social} & 1.69  & 2.25  & 3.30  & 3.89  & 4.69  \\
    CS-LSTM~\cite{deo2018convolutional} & 1.45  & 1.98  & 2.94  & 3.56  & 4.49  \\
    NLS-LSTM~\cite{messaoud2019non} & 0.96  & 1.27  & 2.08  & 2.86  & 3.93\\
     MHA-LSTM~\cite{messaoud2021attention} & 1.25  & 1.48  & 2.57  & 3.22  & 4.20  \\
    CF-LSTM~\cite{xie2021congestion} & 0.72  & 0.91  & 1.73  & 2.59  & 3.44 \\
    STDAN~\cite{chen2022intention} & 0.62  & 0.85  & 1.62  & 2.51  & 3.32  \\
    WSiP~\cite{wang2023wsip} & 0.70  & 0.87  & 1.70  & 2.56  & 3.47  \\
    HLTP~\cite{liao2024cognitive}& 0.55  & 0.76  & 1.44  & 2.39  & 3.27  \\
    BAT~\cite{liao2024bat} &  {0.35}  & 0.74  &  {1.39}  &  {2.19} & {2.88}\\
 \hline
 \textbf{SWIFT} & \textbf{0.33} &  \textbf{0.65} & \textbf{1.19} & \textbf{1.87} & \textbf{2.36} \\
    \toprule
    \end{tabular}%
    }
  \label{mocadresult}%
\end{table}%

\subsubsection{Implementation Details}
The SWIFT model is implemented using the PyTorch framework and trained on four NVIDIA A40 GPUs. Under the map-free experimental setting adopted for the MoCAD and NGSIM datasets, the road features used in the trajectory generation stage are substituted with interaction features.

\subsubsection{Training Details}
The proposed SWIFT is jointly optimized with a composite loss function
\begin{equation}
\label{eq:total-loss}
L = \lambda_\mathrm{a} L_\mathrm{a} + \lambda_\mathrm{b} L_\mathrm{b} + \lambda_\mathrm{c} L_\mathrm{c},
\end{equation}
where $L_\mathrm{a}$, $L_\mathrm{b}$, and $L_\mathrm{c}$ denote the trajectory loss, regime loss, and interaction loss, respectively, and $\lambda_\mathrm{a}$, $\lambda_\mathrm{b}$, $\lambda_\mathrm{c}$ balance the contributions of each term. For nuScenes, we set $\lambda_\mathrm{a}\!=\!1.0$, $\lambda_\mathrm{b}\!=\!0.3$, and $\lambda_\mathrm{c}\!=\!0.3$; for MoCAD and NGSIM, we set $\lambda_\mathrm{a}\!=\!1.0$, $\lambda_\mathrm{b}\!=\!0.5$, and $\lambda_\mathrm{c}\!=\!0.5$. Formal definitions of the three loss functions are provided below.

\noindent\textbf{Trajectory Loss $L_\mathrm{a}$.}
Given the various dataset-specific evaluation protocols, the trajectory loss is defined in a form aligned with conventional metrics for various datasets, enabling fair comparisons with state-of-the-art baselines. Specifically,
\begin{equation}
\label{eq:traj-loss-piecewise}
L_\mathrm{a}^{(d)}=
\begin{cases}
L_\mathrm{ADE}+L_\mathrm{NLL}, & d=\text{nuScenes},\\
L_\mathrm{MSE}+L_\mathrm{NLL}, & d\in\{\text{MoCAD},\text{NGSIM}\}.
\end{cases}
\end{equation}

The losses $L_\mathrm{NLL}$, $L_\mathrm{ADE}$, and $L_\mathrm{MSE}$ are defined as follows. Let $\mathbf{y}^t\!=\!(x^t,y^t)$ be the ground-truth position at time $t$, and $\hat{\mathbf{y}}_{k}^{t}\!=\!(\hat{x}_k^{t},\hat{y}_k^{t})$ the predicted position of the $k$-th trajectory mode with uncertainty scales $\boldsymbol{\mu}_{k}^{t}\!=\!(\mu_{k}^{\mathrm{x},t},\mu_{k}^{\mathrm{y},t})$. Formally,
\begin{equation}
\begin{aligned}
L_\mathrm{NLL}
&=
\frac{1}{T}\sum_{t=1}^{T}
\Big[
\log\!\big(2\boldsymbol{\mu}_{k^\star}^{t}\big)
+
\frac{\big|\mathbf{y}^{t}-\hat{\mathbf{y}}_{k^\star}^{t}\big|}{\boldsymbol{\mu}_{k^\star}^{t}}
\Big]
\\&\quad
-\frac{1}{2}\sum_{k=1}^{K}\tilde{p}_k\log p_k,
\end{aligned}
\end{equation}
where $k^\star=\arg\min_{k} \sum_{t=1}^{T} \big\| \hat{\mathbf{y}}_{k}^{t} - \mathbf{y}^{t} \big\|_2$, $p_k$ denotes the predicted probability of mode $k$, and $\tilde{p}_k$ is a soft target derived from distance-based weighting:
\begin{equation}
\tilde{p}_k \propto 
\exp\!\Big(
-\tfrac{1}{T}\!\sum_{t=1}^{T}\!\|\hat{\mathbf{y}}_{k}^{t}-\mathbf{y}^{t}\|_2
\Big).
\end{equation}

\begin{table}[tbp]
  \centering
  \caption{Experimental results on NGSIM. Metric: RMSE}
  \setlength{\tabcolsep}{3mm}
   \resizebox{0.9\linewidth}{!}{
    \begin{tabular}{cccccc}
    \bottomrule
\rowcolor{gray!15}{Model} & {1\,s} & {2\,s} & {3\,s} & {4\,s} & {5\,s} \\
    \hline
   S-GAN~\cite{gupta2018social} & 0.57  & 1.32  & 2.22  & 3.26  & 4.40 \\
           CS-LSTM~\cite{deo2018convolutional} & 0.61  & 1.27  & 2.09  & 3.10  & 4.37 \\
           MATF-GAN~\cite{zhao2019multi} & 0.66  & 1.34  & 2.08  & 2.97  & 4.13 \\
           NLS-LSTM~\cite{messaoud2019non} & 0.56  & 1.22  & 2.02  & 3.03  & 4.30\\
       WSiP~\cite{wang2023wsip} & 0.56  & 1.23  & 2.05  & 3.08  & 4.34 \\
       CF-LSTM~\cite{xie2021congestion} & 0.55  & 1.10  & 1.78  & 2.73  & 3.82 \\
       MHA-LSTM~\cite{messaoud2021attention} & 0.41  & 1.01  & 1.74  & 2.67  & 3.83 \\
       STDAN~\cite{chen2022intention} & {0.39}  & 0.96  & 1.61  & 2.56 & 3.67 \\
        iNATran~\cite{chen2022vehicle} & {0.39}  &0.96  &1.61  & 2.42  & 3.43  \\
       FHIF~\cite{zuo2023trajectory} &{0.40}  & 0.98  & 1.66  & 2.52  & 3.63\\ 
    DACR-AMTP~\cite{cong2023dacr}& 0.57  & 1.07  & 1.68  & 2.53  & 3.40 \\ 
    BAT~\cite{liao2024bat}& \textbf{0.23} & {0.81}  & {1.54}  & 2.52 &3.62\\
     GaVa~\cite{liao2024human} & {0.40}  & {0.94}  & {1.52}  & {2.24}  & 3.13  \\
     \hline
   \textbf{SWIFT} & 0.34  & \textbf{0.80} &  \textbf{1.23}  & \textbf{ 1.98 } & \textbf{ 2.75 } \\
    \toprule
    \end{tabular}%
    }
  \label{NGSIMresult}%
\end{table}%s

Let $\mathcal{K}$ be the indices of the top-$K$ modes ranked by possibility $p_k$, the $L_\mathrm{ADE}$ is
\begin{equation}
L_\mathrm{ADE}
=
\min_{k\in\mathcal{K}}
\tfrac{1}{T}\sum_{t=1}^{T}
\|\hat{\mathbf{y}}_{k}^{t}-\mathbf{y}^{t}\|_2.
\end{equation}

The $L_\mathrm{MSE}$ penalizes coordinate deviations on the most probable mode. Let $k^\dagger=\arg\max_{k} p_k$ denote the index of the mode with the largest predicted probability. Formally,
\begin{equation}
L_\mathrm{MSE}
=
\frac{1}{T}\sum_{t=1}^{T}
\big\|\hat{\mathbf{y}}_{k^\dagger}^{t}-\mathbf{y}^{t}\big\|_2^{2}.
\end{equation}

\noindent\textbf{Regime Loss $L_\mathrm{b}$.}
Let $\mathcal{R}=\{\mathrm{f},\mathrm{s},\mathrm{c}\}$ denote the free, synchronized, and congested regimes.
The Flow Regime Encoder outputs regime confidence scores $\mathbf{w}=\{w_r\}_{r\in\mathcal{R}}$ and the predicted probabilities are given by
\begin{equation}
\hat p^\mathrm{reg}_{r}
= \frac{\exp(w_r)}{\sum_{r'\in\mathcal{R}}\exp(w_{r'})},
\qquad r\in\mathcal{R}.
\end{equation}

The regime loss $L_\mathrm{b}$ is the cross-entropy
\begin{equation}
L_\mathrm{b}
= - \sum_{r\in\mathcal{R}} y^\mathrm{reg}_{r}\log \hat p^\mathrm{reg}_{r},
\end{equation}
where $\mathbf{y}^\mathrm{reg}=\{y^\mathrm{reg}_{r}\}_{r\in\mathcal{R}}$ is the one-hot ground-truth label.

\noindent\textbf{Interaction Loss $L_\mathrm{c}$}
We define an interaction loss aligning the learned edge set $E_\mathrm{le}$ with the optimized edge set $E_\mathrm{op}$. The interaction loss adopts a binary cross-entropy form:
\begin{equation}
\begin{aligned}
L_\mathrm{c}
&=-\tfrac{1}{N^2}\sum_{i=1}^{N}\sum_{j=1}^{N}
\big[
E_\mathrm{op}(i,j)\log E_\mathrm{le}(i,j)
\\&\qquad
+(1-E_\mathrm{op}(i,j))\log(1-E_\mathrm{le}(i,j))
\big].
\end{aligned}
\end{equation}

To enable end-to-end learning through the binarization in $E_\mathrm{le}$, we adopt the Straight-Through Estimator~\cite{bengio2013estimating}.

\begin{table}[tbp]
\centering
\caption{Sample Efficiency Analysis on the MoCAD.}
\label{tab:sample_efficiency}
  \setlength{\tabcolsep}{3mm}
  \resizebox{0.95\linewidth}{!}{
\begin{tabular}{cccccc}
\bottomrule
\rowcolor{gray!15}{Model} & {1\,s} & {2\,s} & {3\,s} & {4\,s} & {5\,s} \\
\hline
SWIFT           & 0.33 & 0.65 & 1.19 & 1.87 & 2.36 \\
SWIFT (60\%)    & 0.41 & 0.77 & 1.33 & 2.13 & 2.63 \\
Drop (\%)       & 24.2\% & 18.5\% & 11.8\% & 13.9\% & 11.4\% \\
\midrule
HLTP~\cite{liao2024cognitive} & 0.55 & 0.76 & 1.44 & 2.39 & 3.27 \\
HLTP (60\%)     & 0.72 & 1.04 & 1.75 & 2.89 & 4.09 \\
Drop (\%)       & 30.9\% & 36.8\% & 21.5\% & 20.9\% & 25.1\% \\
\toprule
\end{tabular}
}
\end{table}

\begin{table}[t]
\centering
\caption{Robustness comparison on the NGSIM.}
\label{tab:robustness}
  \setlength{\tabcolsep}{3mm}
   \resizebox{0.95\linewidth}{!}{
\begin{tabular}{c|c|ccccc}
\bottomrule
\rowcolor{gray!15} {Noise} & Model &{1\,s} & {2\,s} & {3\,s} & {4\,s} & {5\,s} \\
\hline
\multirow{3}{*}{$\rho =0$} 
   & BAT\cite{liao2024bat} & 0.23 & {0.81}  & {1.54}  & 2.52 &3.62\\
   &  GaVa\cite{liao2024human}  & {0.40}  & {0.94}  & {1.52}  & {2.24}  & 3.13  \\
   &SWIFT & 0.34  & 0.80 &  1.23  &  1.98  &  2.75  \\
\hline
\multirow{3}{*}{$\rho =3$} 
   & BAT~\cite{liao2024bat}      & 0.28 & 0.90 & 1.65 & 2.68 & 3.82 \\
   & GaVa\cite{liao2024human}      & 0.45 & 1.05 & 1.64 & 2.43 & 3.39 \\
   & SWIFT     & 0.37 & 0.87 & 1.32 & 2.05 & 2.83 \\
\hline
\multirow{3}{*}{$\rho =6$} 
   & BAT~\cite{liao2024bat}      & 0.45 & 1.15 &  2.18 & 3.32 & 4.34 \\
   & GaVa\cite{liao2024human}    & 0.53 & 1.28 & 2.00 & 3.10 & 4.05 \\
   & SWIFT     & 0.42 & 0.97 & 1.48 & 2.26 & 3.12 \\
\toprule
\end{tabular}}
\end{table}

\subsection{Experimental Results}
\subsubsection{Predictive Accuracy} We summarize all quantitative results in Tables \ref{table:performance_nuscene}, \ref{mocadresult}, and \ref{NGSIMresult}, corresponding to the nuScenes, MoCAD, and NGSIM datasets, respectively. On the nuScenes dataset, SWIFT achieves substantial improvements over the best baseline WAKE. Specifically, it reduces $\text{minADE}_5$ from 1.24 to 1.15 (7.2\% improvement), $\text{minADE}_1$ from 3.03 to 2.89 (4.6\% improvement), and $\text{minFDE}_1$ from 7.02 to 6.73 (4.1\% improvement). On the MoCAD dataset, which features dense and heterogeneous urban/campus traffic, SWIFT consistently outperforms all prior methods across all prediction horizons. Compared to the previous best model BAT, SWIFT achieves an 18.1\% reduction in RMSE at 5\,s, along with 5.7\%, 12.2\%, 14.4\%, and 14.6\% improvements at 1\,s, 2\,s, 3\,s, and 4\,s, respectively. On the NGSIM dataset, which reflects structured, high-speed highway scenarios, SWIFT delivers particularly strong long-term prediction performance. While BAT achieves the best 1-second RMSE (0.23), SWIFT outperforms all baselines from 2\,s onward. For instance, at 3 seconds, RMSE drops from 1.52 (GaVa) to 1.23 (19.1\% improvement), and at 5\,s, from 3.13 to 2.75 (12.1\% improvement). Collectively, these results confirm SWIFT's suitability as a unified trajectory prediction solution across diverse real-world driving contexts.

\begin{table}[t]
\centering
\caption{Cross-location generalization on the NGSIM dataset (train on I-80, test on US-101). Reported values are RMSE over different prediction horizons.}
\label{tab:crossdatasetI80}
 \setlength{\tabcolsep}{3.5mm}
\resizebox{0.95\linewidth}{!}{
\begin{tabular}{cccccc}
\bottomrule
\rowcolor{gray!15}{Model} &{1\,s} &{2\,s} &{3\,s} &{4\,s} &{5\,s} \\
\hline
STDAN~\cite{chen2022intention} & 0.41 & 1.05 & 1.90 & 2.99 & 4.37\\
GaVa~\cite{liao2024human}     & 0.44 & 1.02 & 1.87 & 2.69 & 3.81 \\
\textbf{SWIFT}                & 0.36 & 0.92 & 1.40 & 2.24 & 3.21 \\
\toprule
\end{tabular}}
\end{table}

\begin{table}[t]
\centering
\caption{Cross-location generalization on the NGSIM dataset (train on US-101, test on I-80). Reported values are RMSE over different prediction horizons.}
\label{tab:crossdatasetUS101}
 \setlength{\tabcolsep}{3.5mm}
\resizebox{0.95\linewidth}{!}{
\begin{tabular}{cccccc}
\bottomrule
\rowcolor{gray!15}{Model} &{1\,s} &{2\,s} &{3\,s} &{4\,s} &{5\,s} \\
\hline
STDAN~\cite{chen2022intention} & 0.48 & 1.13 & 1.93 & 2.94 & 4.22\\
GaVa~\cite{liao2024human}     & 0.45 & 1.09 & 1.90 & 2.63 & 3.77 \\
 \textbf{SWIFT}                & 0.35 & 0.93 & 1.37 & 2.21 & 3.18 \\
\hline
\end{tabular}}
\end{table}

\subsubsection{Sample Efficiency}
The sample efficiency of models determines their practical viability given the high costs of real-world data collection. To evaluate sample efficiency, we train models on a 60\% subset of the MoCAD dataset and present their performance on the test set in Table~\ref{tab:sample_efficiency}. These models include SWIFT and HLTP~\cite{liao2024cognitive}. 
As illustrated in Table~\ref{tab:sample_efficiency}, SWIFT (60\%) maintains strong performance with only moderate degradation. Across prediction horizons from 1 to 5\,s, SWIFT exhibits relative RMSE increases ranging from 11.4\% to 24.2\%. In contrast, HLTP suffers substantially larger performance drops, with relative increases between 20.9\% and 36.8\%. Notably, beyond the 2\,s prediction horizon, SWIFT's degradation consistently remains below 20\%, highlighting its superior robustness under limited data conditions. For example, at a 5\,s horizon, SWIFT trained with 60\% of the data incurs only an 11.4\% increase in RMSE compared to its full-data counterpart, whereas HLTP exhibits a considerable increase of 25.1\%.
These findings collectively support our hypothesis that the structural priors embedded in the SWIFT introduce effective inductive biases, thereby reducing the model’s dependence on large-scale training data.

\begin{table}[t]
  \centering
  \caption{Ablation study of components for trajectory prediction accuracy on the nuScenes dataset. IR, SE, DE, and HD refer to the Interaction Reasoning module, Static Encoder, Dynamic Encoder, and Hybrid Decoder.}
  \label{tab:interaction_ablation}
    \setlength{\tabcolsep}{2mm}
  \resizebox{0.95\linewidth}{!}{
  \begin{tabular}{cccc|ccc}
  \rowcolor{gray!15} 
  \bottomrule
    {IR} & {SE} & {DE} & {HD} & {minADE}$_5$ & {minADE}$_1$ & {minFDE}$_1$ \\
    \hline
    \checkmark &  & \checkmark & \checkmark & 1.49 & 3.32 & 7.90 \\
    \checkmark & \checkmark &        &   \checkmark & 1.30 & 3.26 & 7.58 \\
    \checkmark &        &        & \checkmark & 1.61 & 3.53 & 8.43 \\
     & \checkmark & \checkmark &   \checkmark     & 1.27 & 3.15 & 7.37 \\
    \checkmark & \checkmark & \checkmark & \checkmark & \textbf{1.15} & \textbf{2.89} & \textbf{6.73} \\
    \hline
  \end{tabular}}
\end{table}

\begin{table}[t]
  \centering
  \caption{Ablation study of fusion strategy in the hybrid decoder on the nuScenes dataset.}
  \label{tab:ablation_fusion}
    \setlength{\tabcolsep}{3.7mm}
  \resizebox{0.95\linewidth}{!}{
  \begin{tabular}{c|ccc}
    \bottomrule
  \rowcolor{gray!15} {Strategy} & {minADE}$_5$ & {minADE}$_1$ & {minFDE}$_1$ \\
    \hline
    Add & 1.29 & 3.18 & 7.39 \\
    Concatenation & 1.27 & 3.23 & 7.49 \\
    Attention & \textbf{1.15} & \textbf{2.89} & \textbf{6.73} \\
    \toprule
  \end{tabular}}
\end{table}

\begin{table}[t]
  \centering
  \caption{Ablation study for robustness and sample efficiency on the nuScenes dataset. FRE and SWIN refer to the Flow Regime Encoder and Small-world Interaction Network.}
  \label{tab:struct_setting_summary}
      \setlength{\tabcolsep}{0.5mm}
  \resizebox{0.95\linewidth}{!}{
  \begin{tabular}{cc|cc|ccc}
    \bottomrule
      \rowcolor{gray!15} {FRE} & {SWIN} & {Train (\%)} & {Noise ($\rho$)} & {minADE}$_5$& {minADE}$_1$ &{minFDE}$_1$ \\
    \hline
               &                  & 60 & 0 & 1.35 & 3.37 & 7.79 \\
               & \checkmark       & 60 & 0 & 1.27 & 3.23 & 7.54 \\
    \checkmark & \checkmark       & 60 & 0 & 1.23 & 3.15 & 7.46 \\
    
               &                  & 100 & 6 & 1.61 & 3.50 & 8.36 \\
               &                  & 100 & 3 & 1.39 & 3.34 & 7.59 \\
    \checkmark & \checkmark       & 100 & 6 & 1.31 & 3.18 & 7.26 \\
    
    \checkmark & \checkmark       & 100 & 0 & 1.15 & 2.89 & 6.73 \\
    \toprule
  \end{tabular}
  }
\end{table}

\begin{table}[t]
  \centering
  \caption{Ablation study of exponent $\varphi$ on the NGSIM.}
  \label{tab:connectioncost}
   \setlength{\tabcolsep}{4.3mm}
  \resizebox{0.95\linewidth}{!}{
  \begin{tabular}{c|ccccc}
    \bottomrule
   \rowcolor{gray!15}   $\varphi$  & {1\,s} & {2\,s} & {3\,s} & {4\,s} & {5\,s} \\
    \hline
    1 & \textbf{0.31} & 0.85 & 1.34 & 2.04 & 2.84 \\
    2 & 0.36 & 0.85 & 1.32 & 1.99 & 2.81 \\
    3 & 0.34  & \textbf{0.80} &  \textbf{1.23}  & \textbf{ 1.98 } & \textbf{ 2.75 } \\
    \toprule
  \end{tabular}}
\end{table}

\subsubsection{Robustness}
In real-world deployments, trajectory data are often corrupted by noise due to sensor inaccuracies and other external factors, whereas training datasets are typically noise-free. This mismatch can lead to a significant drop in prediction performance. Therefore, evaluating a model’s robustness to input noise is critical. To this end, we inject controlled noise into the observed trajectories, where the noise intensity is governed by a scalar parameter $\rho$. We evaluate model performance on the NGSIM dataset under three levels of observation noise, with the results summarized in Table~\ref{tab:robustness}. As illustrated in Table~\ref{tab:robustness}, the performance of GaVa and BAT deteriorates more noticeably under increasing noise levels, whereas SWIFT maintains relatively stable accuracy. As $\rho$ increases from 0 to 6, the performance gap between SWIFT and the baselines widens, particularly at longer prediction horizons. At the 5\,s horizon under severe noise, SWIFT achieves an RMSE of 3.12, compared to 4.05 for GaVa and 4.34 for BAT, corresponding to relative improvements of 23.0\% and 39.1\%, respectively. Even under mild noise ($\rho = 3$), SWIFT still outperforms GaVa and BAT by 16.5\% and 25.9\%. Crucially, SWIFT also exhibits significantly less performance degradation across noise levels. From $\rho = 0$ to $\rho = 6$, its RMSE at the 5-second horizon increases by only 13.5\%, whereas GaVa and BAT suffer larger increases of 29.4\% and 19.9\%, respectively. These results validate the robustness of SWIFT under highly corrupted observations.

{
\subsubsection{Generalization}
The ability of trajectory prediction models to generalize to unseen scenarios is crucial for real-world deployment. To assess generalization performance, we conduct cross-location experiments on the NGSIM dataset~\cite{us2008ngsim}, which provides vehicle trajectories from two distinct freeway segments, US-101 and I-80. Models are trained on one location and evaluated on the other, yielding two transfer directions: I-80 $\rightarrow$ US-101 and US-101 $\rightarrow$ I-80. This cross-location protocol naturally removes temporal correlation between the training and test sets, enabling a more reliable evaluation of generalization performance.
As reported in Table~\ref{tab:crossdatasetI80}, when transferring from I-80 to US-101, SWIFT consistently achieves the lowest RMSE across all prediction horizons. In particular, at the 5\,s horizon, SWIFT attains an RMSE of 3.21, outperforming GaVa (3.81) and STDAN (4.37) by 15.8\% and 26.6\%, respectively. A similar trend is observed in the opposite direction (Table~\ref{tab:crossdatasetUS101}), where SWIFT yields an RMSE of 3.18 at 5\,s, compared with 3.77 for GaVa and 4.22 for STDAN, corresponding to relative reductions of 15.6\% and 24.6\%. These results demonstrate that SWIFT maintains superior predictive accuracy under spatial and temporal distribution shifts, highlighting its generalization performance.

}

\subsection{Ablation Studies}
\subsubsection{Ablation Study of Prediction Accuracy}
To quantify the contribution of each core component in our model, we conduct ablation experiments on the nuScenes dataset, focusing on the Interaction Reasoning (IR) module, Static Encoder (SE), Dynamic Encoder (DE), and the Hybrid Decoder (HD).
Table~\ref{tab:interaction_ablation} reports the comparison results. Removing any of the modules leads to noticeable performance degradation, demonstrating their contributions. Specifically, disabling the IR module, thereby simplifying interaction modeling to pairwise reasoning, causes minFDE\(_1\) to rise from 6.73 to 7.37, a 9.5\% increase, indicating the importance of capturing higher-order relational patterns. Excluding the DE results in a significant drop in accuracy, with minADE\(_1\) increasing by 12.8\%, highlighting the crucial role of motion dynamics. The absence of the SE, which removes access to HD map context, increases minADE\(_5\) by 29.5\%, reflecting the value of spatial priors in guiding multi-modal predictions. These results validate the complementary nature of the modules and underscore their collective importance for accurate trajectory prediction.

\begin{table}[t]
  \centering
  \caption{Ablation study of loss weights on the MoCAD.}
  \label{tab:lossweights_abalation}
 \setlength{\tabcolsep}{3mm}
  \resizebox{0.95\linewidth}{!}{
  \begin{tabular}{ccc|ccccc}
    \bottomrule
    \rowcolor{gray!15} $\lambda_\mathrm{a}$ & $\lambda_\mathrm{b}$ & $\lambda_\mathrm{c}$ &{1\,s} & {2\,s} & {3\,s} & {4\,s} & {5\,s} \\
    \hline
    1.0 & 0.1  & 0.1 & \textbf{0.30} & 0.68 & \textbf{1.19} & 1.91 & 2.47 \\
    0.5 & 0.5  & 0.5 & 0.38 & 0.72 & 1.25 & 2.11 & 2.62 \\
    1.0 & 0.5  & 0.5 & 0.33 &  \textbf{0.65} & \textbf{1.19} & \textbf{1.87} & \textbf{2.36} \\
    \toprule
  \end{tabular}}
\end{table}

\begin{table}[t]
  \centering
  \caption{Optimal range of interaction parameters $\alpha$, $\beta$, and $\gamma$ across distinct traffic regimes.}
  \label{tab:flowadapt}
 \setlength{\tabcolsep}{4.5mm}
  \resizebox{0.95\linewidth}{!}{
  \begin{tabular}{c|ccc}
    \bottomrule
         \rowcolor{gray!15} Regime  & $\alpha$ & $\beta$ & $\gamma$   \\
    \hline
    Free  Flow & [0.4-0.5] & [0.5-0.6] & [0.6-0.7]  \\
    Sync. Flow & [0.7-0.8] & [0.7-0.8] & [0.3-0.4]  \\
    Cong. Flow & [0.8-0.9] & [0.4-0.5] & [0.4-0.5]  \\
    \toprule
  \end{tabular}}
\end{table}

{
\subsubsection{Ablation Study of Fusion Strategy}
To assess the effectiveness of the attention-based fusion strategy in the hybrid decoder, we conduct an ablation study by replacing the attention-based fusion strategy with two fusion alternatives: (1) Concatenation, which concatenates the static query $\mathcal{Q}_\mathrm{s}$ and the dynamic query $\mathcal{Q}_\mathrm{d}$, and (2) Add, which performs element-wise addition of the static and dynamic queries. The results are summarized in Table~\ref{tab:ablation_fusion}.
Compared with both alternatives, the attention-based fusion consistently achieves superior performance across all metrics, with minADE\(_1\) and minFDE\(_1\) reduced by 10.5\% and 10.1\%, respectively, over concatenation. 
Overall, these results demonstrate that the proposed attention-based hybrid decoder effectively balances static and dynamic contexts, leading to more accurate trajectory predictions.

}

\subsubsection{Ablation Study of Structural Priors}
To assess the impact of structural priors on model robustness and sample efficiency, we conduct ablation experiments on the nuScenes dataset. Specifically, we evaluate the contributions of the Flow Regime Encoder (FRE) and the Small-World Interaction Network (SWIN) under two challenging conditions: limited training data (60\% of the full set) and noisy input trajectories with different noise levels (\(\rho = 3, 6\)).
As shown in Table~\ref{tab:struct_setting_summary}, removing both FRE and SWIN results in the weakest performance across all conditions. Under limited data (60\%), adding SWIN alone reduces minADE\(_5\) from 1.35 to 1.27 (a 5.9\% improvement), while incorporating FRE brings it further down to 1.23 (an additional 3.1\% gain). Similarly, minFDE\(_1\) improves from 7.79 to 7.46, indicating better prediction fidelity even with reduced supervision.
Under strong input noise (\(\rho = 6\)), the full model achieves a 15.1\% reduction in minFDE\(_1\) (from 8.36 to 7.26) and a 10.1\% drop in minADE\(_1\) (from 3.50 to 3.18), showing increased robustness. 
These results confirm that both FRE and SWIN contribute meaningfully and complementarily to performance. FRE introduces global traffic condition awareness, while SWIN facilitates long-range interaction modeling. Together, they significantly enhance the model’s generalization, noise resilience, and sample efficiency.

{
\subsubsection{Ablation Study of Connection-Cost Exponent $\varphi$}
\label{connectioncost}
To assess the sensitivity of the structural optimization to the connection-cost exponent $\varphi$, we vary $\varphi \in \{1,2,3\}$ in Eq.~\eqref{costfunction} on NGSIM. As shown in Table~\ref{tab:connectioncost}, $\varphi\!=\!1$ attains the lowest error at 1\,s prediction horizon, but its performance degrades for longer horizons, being 3.2\% worse than $\varphi\!=\!3$ at 5\,s. Overall, $\varphi\!=\!3$ achieves the best performance.

}

{
\subsubsection{Ablation Study of Loss Weights}
To rigorously evaluate the contribution and robustness of each component in the composite loss function Eq.\eqref{eq:total-loss}, we vary the loss weights $\lambda_\mathrm{a}$, $\lambda_\mathrm{b}$, and $\lambda_\mathrm{c}$ on the MoCAD dataset. 
As summarized in Table~\ref{tab:lossweights_abalation}, the configuration ($\lambda_\mathrm{a}\!=\!1.0$, $\lambda_\mathrm{b}\!=\!0.5$, $\lambda_\mathrm{c}\!=\!0.5$) achieves the best performance across most prediction horizons. 
When the contributions of the regime and interaction losses are increased relative to the trajectory loss, i.e., in the configuration ($\lambda_\mathrm{a}\!=\!0.5$, $\lambda_\mathrm{b}\!=\!0.5$, $\lambda_\mathrm{c}\!=\!0.5$), performance degrades noticeably, with the 5\,s prediction error rising by 11.0\%. 
Conversely, when the trajectory loss is overly emphasized ($\lambda_\mathrm{a}\!=\!1.0$, $\lambda_\mathrm{b}\!=\!0.1$, $\lambda_\mathrm{c}\!=\!0.1$), the model shows a slight advantage at short horizons but suffers a 4.6\% drop in accuracy at 5\,s. 
These results demonstrate that SWIFT benefits from the complementary regularization of all three loss terms, leading to stable and accurate trajectory prediction.

}

\begin{figure}[t]
    \centering
    \includegraphics[width=0.96\linewidth]{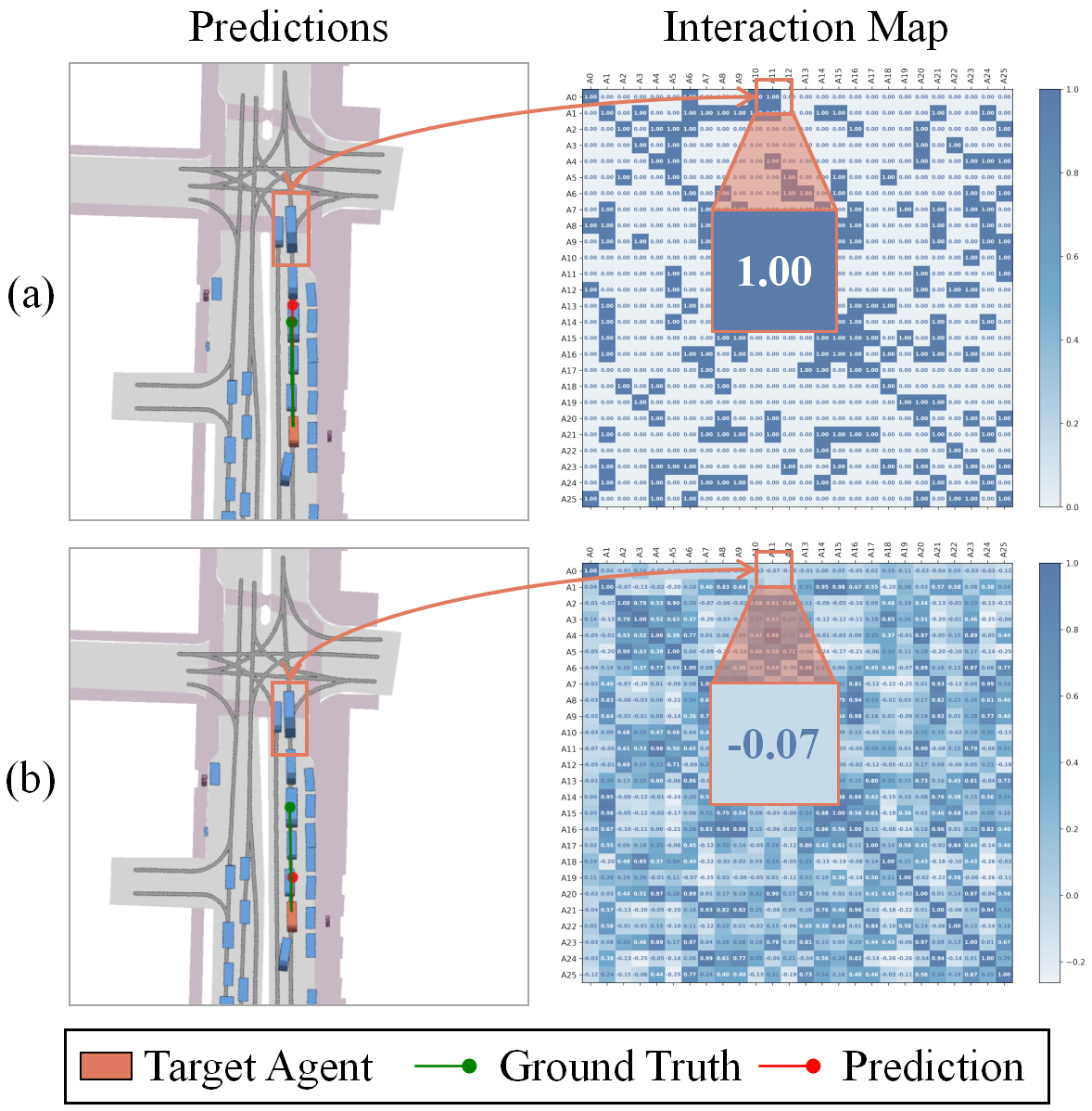}
    \caption{Qualitative comparison of trajectory prediction and interaction modeling between SWIFT (a) and SWIFT* (b) in a complex intersection scenario. In (a), SWIFT assigns high attention to a distant accelerating vehicle (labeled with an orange box), leveraging it as a global motion cue and accurately predicting the target agent’s forward motion. In (b), SWIFT* fails to capture this global dependency and produces an overly conservative prediction. This case highlights the critical role of modeling global interaction in trajectory prediction.}
    \label{fig:global_case}
\end{figure}

\begin{figure*}[t]
        \centering
	\includegraphics[width=0.83\linewidth]{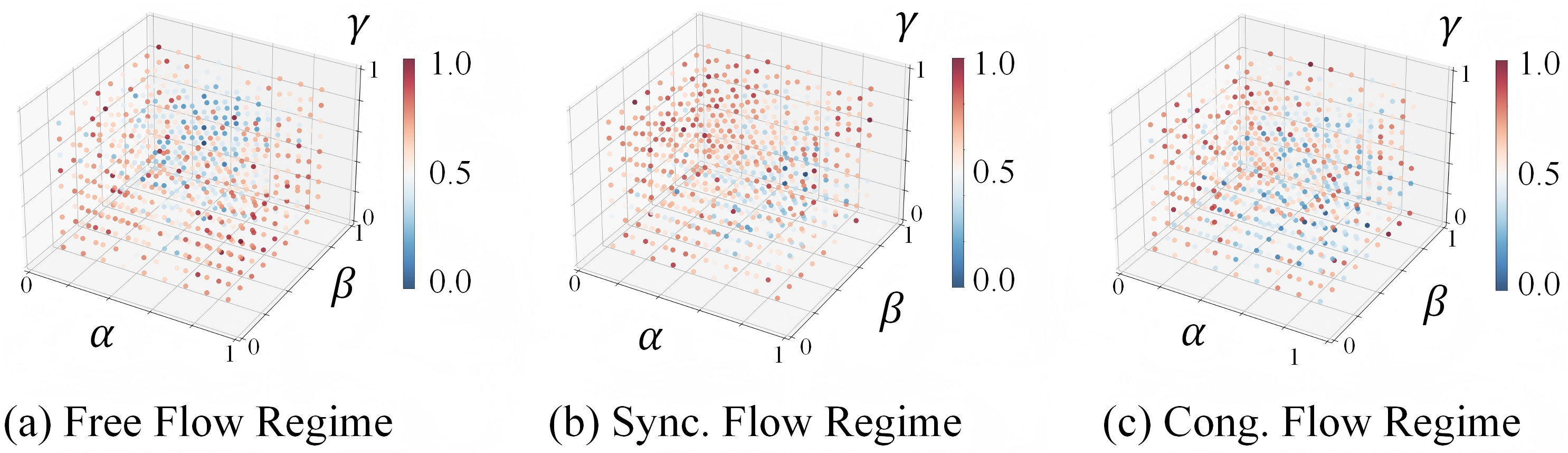}
	\caption{Grid search results of the interaction parameters $(\alpha, \beta, \gamma)$ across three traffic regimes. Each point denotes one parameter combination, and color indicates the normalized prediction error ($\text{minADE}_1$), with blue representing lower prediction error. Notably, the low-error (blue) points exhibit markedly different distribution patterns across regimes.}
    \label{flowcase}
\end{figure*}

\begin{figure}[t]
    \centering
    \includegraphics[width=\linewidth]{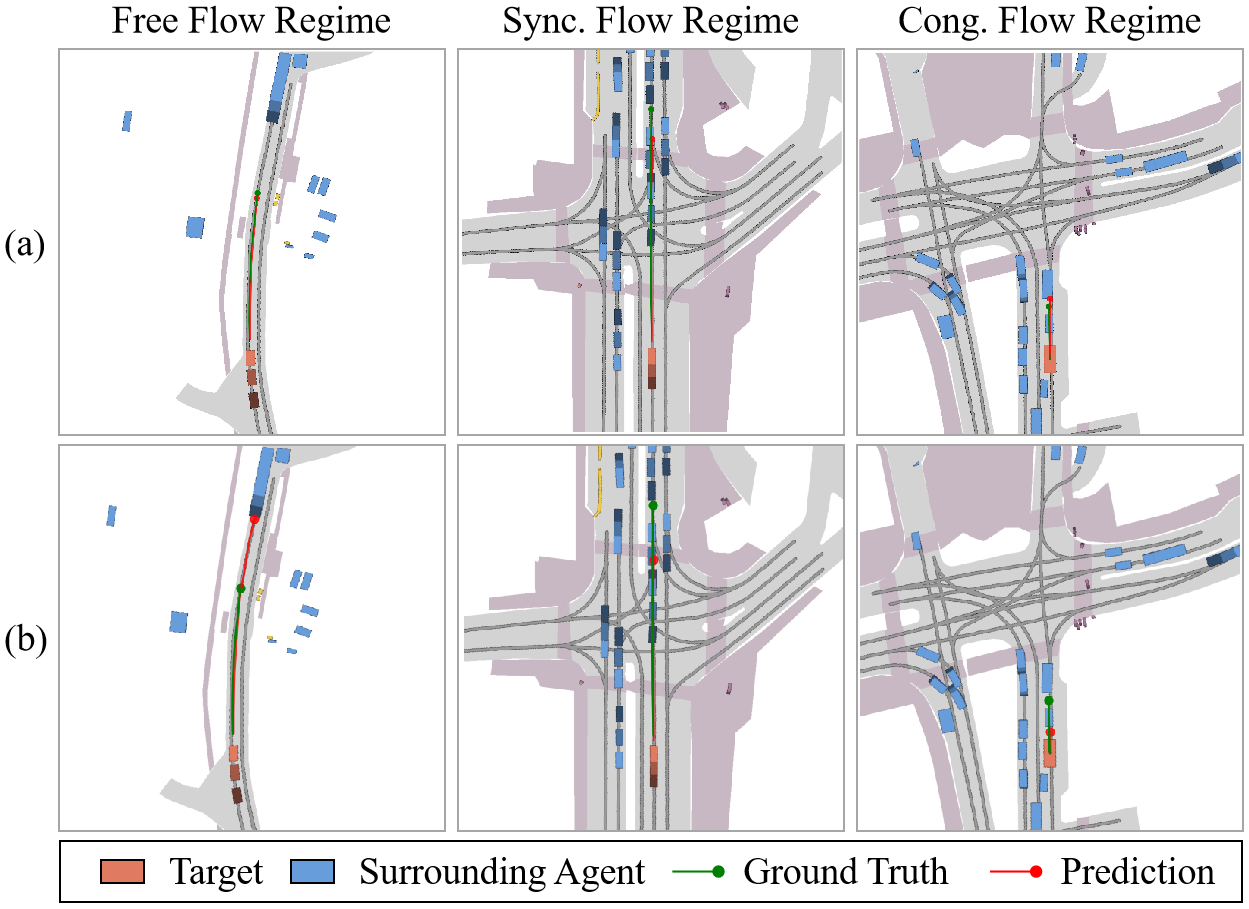}
    \caption{
    Qualitative comparison between SWIFT and its ablated variant SWIFT* across three distinct traffic flow regimes: free flow, synchronized flow, and congested flow. As illustrated, SWIFT provides more accurate and context-sensitive predictions across all three regimes.
    }
    \label{case2}
\end{figure}

\begin{figure}[t]
    \centering
    \includegraphics[width=\linewidth]{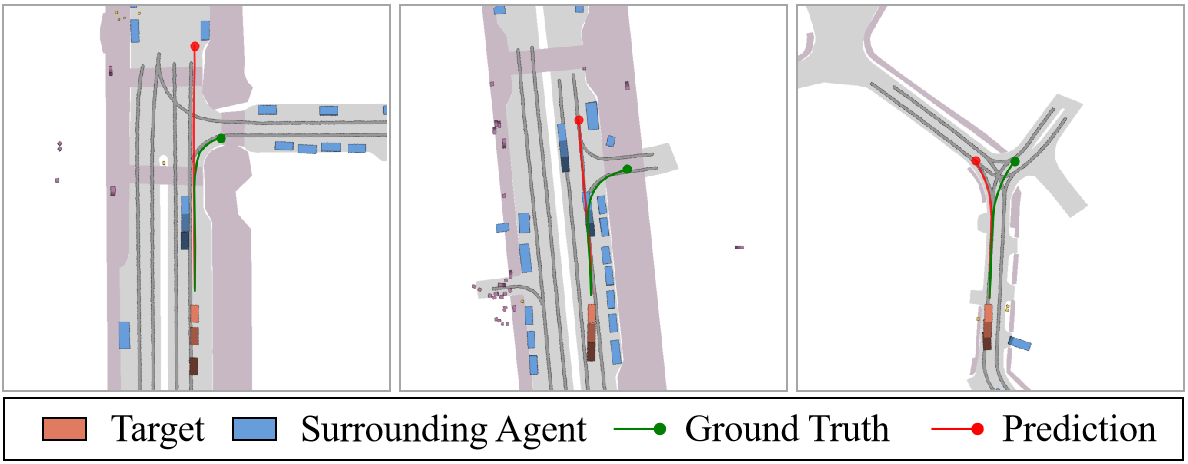}
    \caption{Failure cases of SWIFT. These failures arise when the target agent executes late maneuvers, such as unexpected lane changes or turns. While SWIFT captures early motion cues well, it struggles to anticipate shifts in long-term intentions that are not observable from immediate history.}
    \label{Failure}
\end{figure}

\subsection{Case Studies}

\subsubsection{Global Interaction}
\label{sec:varphi_ablation}
To qualitatively assess the impact of global interaction modeling, we compare the predicted trajectories and interaction maps produced by SWIFT and SWIFT*, an ablated variant, in a representative intersection scenario. In the SWIFT*, the interaction map is constructed solely based on agent-level similarity, thereby removing the small-world interaction network and flow regime encoder. As shown in Fig.~\ref{fig:global_case}, the target agent is initially stationary behind two other agents, while a leading agent farther ahead has begun to accelerate. Although this leading agent is not in close spatial proximity, its motion is a key global cue that signals the likely onset of forward movement. SWIFT accurately anticipates this behavior by assigning high interaction weights to the distant agent, resulting in a predicted trajectory that closely matches the ground truth. In contrast, SWIFT*, constrained to local interactions, fails to capture the anticipatory signal and produces a significantly shorter prediction. This case illustrates the importance of incorporating global dependencies for reliable prediction in traffic environments.

{
\subsubsection{Flow Regime Adaptation}
To investigate how the proposed SWIFT framework adapts its interaction mechanism under distinct traffic conditions, we perform a flow regime-wise analysis on the nuScenes dataset. For each regime, we randomly sample 100 representative scenes with a regime confidence score greater than 0.6, ensuring reliable classification of traffic regimes. A grid search is conducted over the interaction parameters $(\alpha, \beta, \gamma)$ within the range of $[0.1, 0.9]$ with a step size of $0.1$, and the normalized $\mathrm{minADE}_1$ is computed for each parameter configuration.
Fig.~\ref{flowcase} and Table~\ref{tab:flowadapt} jointly summarize the qualitative and quantitative results, revealing clear regime-dependent patterns. In the free flow regime, the optimal interaction threshold $\gamma$ lies in a higher range, indicating that vehicle motion is less influenced by neighboring agents and that interactions are relatively sparse. In contrast, the synchronized flow regime is characterized by a lower optimal $\gamma$ and comparable local and global weights, suggesting that both local and global interactions frequently coexist. In the congested flow regime, the $\alpha$ value implies that, as vehicle speeds decrease and spacing becomes minimal, agents primarily respond to their immediate neighbors.
Together, these analyses underscore the necessity of incorporating flow-aware adaptation in trajectory prediction models, as interaction patterns inherently depend on the traffic regime.

In addition to the grid search experiment, we conduct a qualitative comparison between SWIFT and an ablated variant, denoted SWIFT*, across three representative scenes corresponding to free flow, synchronized flow, and congested flow. SWIFT* removes the Flow Regime Encoder, resulting in a static interaction graph that does not respond to macroscopic traffic states. As shown in Fig.~\ref{case2}, SWIFT yields trajectory predictions that are more accurate and more consistent with the prevailing regime, whereas SWIFT* often appears conservative in free flow and misaligned in synchronized and congested conditions, reflecting the limitations of a fixed interaction topology. These comparison further highlights the effectiveness of the proposed Flow Regime Encoder.

}

\subsubsection{Failure Scenarios}
While SWIFT demonstrates strong capability in capturing motion dynamics and modeling interactions, it may still fail in scenarios where the future behavior of a driver involves abrupt or delayed intention changes. As shown in Fig.~\ref{Failure}, the predicted trajectories closely follow the ground truth in the early stages but gradually diverge when the target vehicle performs late merges, lane changes, or turns. These deviations reflect the inherent difficulty of inferring future intention shifts from limited current observations. This underscores the need for autonomous driving systems to perform continuous, real-time prediction updates in order to adapt to evolving driver behaviors and reduce the impact of latent intention uncertainty in complex traffic scenarios.

\section{Conclusion}\label{Conclusion}
In this work, we present SWIFT, a unified framework for trajectory prediction that integrates structural priors from Small-World Networks and Traffic Flow Theory. Motivated by the limitations of existing approaches, particularly their reliance on local, pairwise interactions and their lack of scene-level flow awareness, SWIFT rethinks the modeling of inter-agent interactions through a structured, flow-adaptive lens. Extensive experiments across three real-world datasets, including nuScenes, MoCAD, and NGSIM, demonstrate that SWIFT consistently outperforms state-of-the-art methods in terms of prediction accuracy, generalization to unseen scenarios, and robustness under noisy inputs. Comprehensive ablation and case studies reinforce the effectiveness of structurally informed, flow-adaptive interaction modeling. These results highlight the importance of integrating structural priors into trajectory prediction models for improved robustness, adaptability, and interpretability. Nonetheless, our work remains challenged by abrupt intention shifts that are not observable from recent motion cues, suggesting future extensions toward long-horizon intent reasoning. Overall, this work highlights the importance of structural inductive biases in interaction modeling and outlines a promising direction for robust and generalizable trajectory prediction in complex traffic environments.

\bibliographystyle{IEEEtran}
% argument is your BibTeX string definitions and bibliography database(s)
%\bibliography{IEEEabrv,../bib/paper}

%\bibliographystyle{mybstfile}
\bibliography{mybibfile}

% Generated by IEEEtran.bst, version: 1.14 (2015/08/26)
\begin{thebibliography}{10}
\providecommand{\url}[1]{#1}
\csname url@samestyle\endcsname
\providecommand{\newblock}{\relax}
\providecommand{\bibinfo}[2]{#2}
\providecommand{\BIBentrySTDinterwordspacing}{\spaceskip=0pt\relax}
\providecommand{\BIBentryALTinterwordstretchfactor}{4}
\providecommand{\BIBentryALTinterwordspacing}{\spaceskip=\fontdimen2\font plus
\BIBentryALTinterwordstretchfactor\fontdimen3\font minus \fontdimen4\font\relax}
\providecommand{\BIBforeignlanguage}[2]{{%
\expandafter\ifx\csname l@#1\endcsname\relax
\typeout{** WARNING: IEEEtran.bst: No hyphenation pattern has been}%
\typeout{** loaded for the language `#1'. Using the pattern for}%
\typeout{** the default language instead.}%
\else
\language=\csname l@#1\endcsname
\fi
#2}}
\providecommand{\BIBdecl}{\relax}
\BIBdecl

\bibitem{marchetti2024smemo}
F.~Marchetti, F.~Becattini, L.~Seidenari, and A.~Del~Bimbo, ``Smemo: social memory for trajectory forecasting,'' \emph{IEEE Transactions on Pattern Analysis and Machine Intelligence}, vol.~46, no.~6, pp. 4410--4425, 2024.

\bibitem{sun2023modality}
J.~Sun, Y.~Li, L.~Chai, and C.~Lu, ``Modality exploration, retrieval and adaptation for trajectory prediction,'' \emph{IEEE Transactions on Pattern Analysis and Machine Intelligence}, vol.~45, no.~12, pp. 15\,051--15\,064, 2023.

\bibitem{zhang2024decoupling}
B.~Zhang, N.~Song, and L.~Zhang, ``Decoupling motion forecasting into directional intentions and dynamic states,'' \emph{NIPS}, 2024.

\bibitem{liu2025end}
C.~Liu, J.~Xie, F.~Chang, T.~Li, S.~Li, and Y.~Wang, ``End-point drive and reverse enhanced decoding-based traffic participants trajectory prediction under bird’s eye view,'' \emph{IEEE Transactions on Intelligent Transportation Systems}, 2025.

\bibitem{liao2025sa}
H.~Liao, Z.~Li, K.~Zhu, K.~Li, and C.~Xu, ``Sa-tp: A safety-aware trajectory prediction and planning model for autonomous driving,'' \emph{IEEE Transactions on Robotics}, 2025.

\bibitem{du2024social}
Q.~Du, X.~Wang, S.~Yin, L.~Li, and H.~Ning, ``Social force embedded mixed graph convolutional network for multi-class trajectory prediction,'' \emph{IEEE Transactions on Intelligent Vehicles}, 2024.

\bibitem{mo2024heterogeneous}
X.~Mo, Y.~Xing, and C.~Lv, ``Heterogeneous graph social pooling for interaction-aware vehicle trajectory prediction,'' \emph{Transportation Research Part E: Logistics and Transportation Review}, vol. 191, p. 103748, 2024.

\bibitem{ghavasieh2024diversity}
A.~Ghavasieh and M.~De~Domenico, ``Diversity of information pathways drives sparsity in real-world networks,'' \emph{Nature Physics}, vol.~20, no.~3, pp. 512--519, 2024.

\bibitem{kerner2009introduction}
B.~S. Kerner, \emph{Introduction to modern traffic flow theory and control: the long road to three-phase traffic theory}.\hskip 1em plus 0.5em minus 0.4em\relax Springer Science \& Business Media, 2009.

\bibitem{kerner2004three}
------, ``Three-phase traffic theory and highway capacity,'' \emph{Physica A: Statistical Mechanics and its Applications}, vol. 333, pp. 379--440, 2004.

\bibitem{liu2024attention}
Y.~Liu, B.~Li, X.~Wang, C.~Sammut, and L.~Yao, ``Attention-aware social graph transformer networks for stochastic trajectory prediction,'' \emph{IEEE Transactions on Knowledge and Data Engineering}, 2024.

\bibitem{chen2022intention}
X.~Chen, H.~Zhang, F.~Zhao, Y.~Hu, C.~Tan, and J.~Yang, ``Intention-aware vehicle trajectory prediction based on spatial-temporal dynamic attention network for internet of vehicles,'' \emph{IEEE Transactions on Intelligent Transportation Systems}, vol.~23, no.~10, pp. 19\,471--19\,483, 2022.

\bibitem{helbing2012social}
D.~Helbing, \emph{Social self-organization: Agent-based simulations and experiments to study emergent social behavior}.\hskip 1em plus 0.5em minus 0.4em\relax Springer, 2012.

\bibitem{chen2025dstigcn}
W.~Chen, H.~Sang, J.~Wang, and Z.~Zhao, ``Dstigcn: Deformable spatial-temporal interaction graph convolution network for pedestrian trajectory prediction,'' \emph{IEEE Transactions on Intelligent Transportation Systems}, 2025.

\bibitem{westny2023mtp}
T.~Westny, J.~Oskarsson, B.~Olofsson, and E.~Frisk, ``Mtp-go: Graph-based probabilistic multi-agent trajectory prediction with neural odes,'' \emph{IEEE Transactions on Intelligent Vehicles}, 2023.

\bibitem{simulation2007us}
G.~SIMulation, ``Us highway 101 dataset,'' 2007.

\bibitem{lan2024hi}
Z.~Lan, Y.~Ren, H.~Yu, L.~Liu, Z.~Li, Y.~Wang, and Z.~Cui, ``Hi-scl: Fighting long-tailed challenges in trajectory prediction with hierarchical wave-semantic contrastive learning,'' \emph{Transportation Research Part C: Emerging Technologies}, vol. 165, p. 104735, 2024.

\bibitem{liao2024physics}
H.~Liao, C.~Wang, Z.~Li, Y.~Li, B.~Wang, G.~Li, and C.~Xu, ``Physics-informed trajectory prediction for autonomous driving under missing observation,'' \emph{International Joint Conference On Artificial Intelligence}, 2024.

\bibitem{alahi2016social}
A.~Alahi, K.~Goel, V.~Ramanathan, A.~Robicquet, L.~Fei-Fei, and S.~Savarese, ``Social lstm: Human trajectory prediction in crowded spaces,'' in \emph{Proceedings of the IEEE conference on computer vision and pattern recognition}, 2016, pp. 961--971.

\bibitem{liao2024cognitive}
H.~Liao, Y.~Li, Z.~Li, C.~Wang, Z.~Cui, S.~E. Li, and C.~Xu, ``A cognitive-based trajectory prediction approach for autonomous driving,'' \emph{IEEE Transactions on Intelligent Vehicles}, 2024.

\bibitem{deo2022multimodal}
N.~Deo, E.~Wolff, and O.~Beijbom, ``Multimodal trajectory prediction conditioned on lane-graph traversals,'' in \emph{Conference on Robot Learning}.\hskip 1em plus 0.5em minus 0.4em\relax PMLR, 2022, pp. 203--212.

\bibitem{chen2024q}
J.~Chen, Z.~Wang, J.~Wang, and B.~Cai, ``Q-eanet: Implicit social modeling for trajectory prediction via experience-anchored queries,'' \emph{IET Intelligent Transport Systems}, vol.~18, no.~6, pp. 1004--1015, 2024.

\bibitem{liao2025toward}
H.~Liao, Z.~Li, G.~Zhang, K.~Li, and C.~Xu, ``Toward human-like trajectory prediction for autonomous driving: A behavior-centric approach,'' \emph{Transportation Science}, 2025.

\bibitem{liang2021nettraj}
Y.~Liang and Z.~Zhao, ``Nettraj: A network-based vehicle trajectory prediction model with directional representation and spatiotemporal attention mechanisms,'' \emph{IEEE Transactions on Intelligent Transportation Systems}, vol.~23, no.~9, pp. 14\,470--14\,481, 2021.

\bibitem{watts1998collective}
D.~J. Watts and S.~H. Strogatz, ``Collective dynamics of ‘small-world’networks,'' \emph{nature}, vol. 393, no. 6684, pp. 440--442, 1998.

\bibitem{lee2024cultural}
Y.~Lee, ``Cultural and opinion dynamics in small-world “social” networks,'' \emph{The Journal of Mathematical Sociology}, pp. 1--26, 2024.

\bibitem{kleinberg2000navigation}
J.~M. Kleinberg, ``Navigation in a small world,'' \emph{Nature}, vol. 406, no. 6798, pp. 845--845, 2000.

\bibitem{zhao2024electric}
A.~P. Zhao, S.~Li, Z.~Li, Z.~Wang, X.~Fei, Z.~Hu, M.~Alhazmi, X.~Yan, C.~Wu, S.~Lu \emph{et~al.}, ``Electric vehicle charging planning: A complex systems perspective,'' \emph{IEEE Transactions on Smart Grid}, 2024.

\bibitem{wu2011shockwave}
X.~Wu and H.~X. Liu, ``A shockwave profile model for traffic flow on congested urban arterials,'' \emph{Transportation Research Part B: Methodological}, vol.~45, no.~10, pp. 1768--1786, 2011.

\bibitem{wang2025nest}
C.~Wang, H.~Liao, B.~Wang, Y.~Guan, B.~Rao, Z.~Pu, Z.~Cui, C.-Z. Xu, and Z.~Li, ``Nest: A neuromodulated small-world hypergraph trajectory prediction model for autonomous driving,'' in \emph{Proceedings of the AAAI Conference on Artificial Intelligence}, vol.~39, no.~1, 2025, pp. 808--816.

\bibitem{liao2024bat}
H.~Liao, Z.~Li, H.~Shen, W.~Zeng, D.~Liao, G.~Li, and C.~Xu, ``Bat: Behavior-aware human-like trajectory prediction for autonomous driving,'' in \emph{Proceedings of the AAAI Conference on Artificial Intelligence}, vol.~38, 2024, pp. 10\,332--10\,340.

\bibitem{zhou2022hivt}
Z.~Zhou, L.~Ye, J.~Wang, K.~Wu, and K.~Lu, ``Hivt: Hierarchical vector transformer for multi-agent motion prediction,'' in \emph{Proceedings of the IEEE/CVF Conference on Computer Vision and Pattern Recognition}, 2022, pp. 8823--8833.

\bibitem{alamdari2023small}
S.~Alamdari, ``Small-world formation via local information,'' \emph{arXiv preprint arXiv:2301.00849}, 2023.

\bibitem{kerner1999physics}
B.~S. Kerner, ``The physics of traffic,'' \emph{Physics world}, vol.~12, no.~8, p.~25, 1999.

\bibitem{bando1995dynamical}
M.~Bando, K.~Hasebe, A.~Nakayama, A.~Shibata, and Y.~Sugiyama, ``Dynamical model of traffic congestion and numerical simulation,'' \emph{Physical review E}, vol.~51, no.~2, p. 1035, 1995.

\bibitem{bando1998analysis}
M.~Bando, K.~Hasebe, K.~Nakanishi, and A.~Nakayama, ``Analysis of optimal velocity model with explicit delay,'' \emph{Physical Review E}, vol.~58, no.~5, p. 5429, 1998.

\bibitem{punzo2016speed}
V.~Punzo and M.~Montanino, ``Speed or spacing? cumulative variables, and convolution of model errors and time in traffic flow models validation and calibration,'' \emph{Transportation Research Part B: Methodological}, vol.~91, pp. 21--33, 2016.

\bibitem{heckbert1995fourier}
P.~Heckbert, ``Fourier transforms and the fast fourier transform (fft) algorithm,'' \emph{Computer Graphics}, vol.~2, no. 1995, pp. 15--463, 1995.

\bibitem{wilson2011car}
R.~E. Wilson and J.~A. Ward, ``Car-following models: fifty years of linear stability analysis--a mathematical perspective,'' \emph{Transportation Planning and Technology}, vol.~34, no.~1, pp. 3--18, 2011.

\bibitem{wilson2008mechanisms}
R.~E. Wilson, ``Mechanisms for spatio-temporal pattern formation in highway traffic models,'' \emph{Philosophical Transactions of the Royal Society A: Mathematical, Physical and Engineering Sciences}, vol. 366, no. 1872, pp. 2017--2032, 2008.

\bibitem{helbing1998coherent}
D.~Helbing and B.~A. Huberman, ``Coherent moving states in highway traffic,'' \emph{Nature}, vol. 396, no. 6713, pp. 738--740, 1998.

\bibitem{kelkar2007extension}
S.~Kelkar, L.~Grigsby, and J.~Langsner, ``An extension of parseval's theorem and its use in calculating transient energy in the frequency domain,'' \emph{IEEE Transactions on Industrial Electronics}, no.~1, pp. 42--45, 2007.

\bibitem{tong2020directed}
Z.~Tong, Y.~Liang, C.~Sun, D.~S. Rosenblum, and A.~Lim, ``Directed graph convolutional network,'' \emph{arXiv preprint arXiv:2004.13970}, 2020.

\bibitem{caesar2020nuscenes}
H.~Caesar, V.~Bankiti, A.~H. Lang, S.~Vora, V.~E. Liong, Q.~Xu, A.~Krishnan, Y.~Pan, G.~Baldan, and O.~Beijbom, ``nuscenes: A multimodal dataset for autonomous driving,'' in \emph{Proceedings of the IEEE/CVF conference on computer vision and pattern recognition}, 2020, pp. 11\,621--11\,631.

\bibitem{us2008ngsim}
U.~D. of~Transportation, ``Ngsim—next generation simulation,'' 2008.

\bibitem{wang2025wake}
C.~Wang, H.~Liao, Z.~Li, and C.~Xu, ``Wake: Towards robust and physically feasible trajectory prediction for autonomous vehicles with wavelet and kinematics synergy,'' \emph{IEEE Transactions on Pattern Analysis and Machine Intelligence}, 2025.

\bibitem{salzmann2020trajectron++}
T.~Salzmann, B.~Ivanovic, P.~Chakravarty, and M.~Pavone, ``Trajectron++: Dynamically-feasible trajectory forecasting with heterogeneous data,'' in \emph{Computer Vision--ECCV 2020: 16th European Conference, Glasgow, UK, August 23--28, 2020, Proceedings, Part XVIII 16}.\hskip 1em plus 0.5em minus 0.4em\relax Springer, 2020, pp. 683--700.

\bibitem{chai2020multipath}
Y.~Chai, B.~Sapp, M.~Bansal, and D.~Anguelov, ``Multipath: Multiple probabilistic anchor trajectory hypotheses for behavior prediction,'' in \emph{Conference on Robot Learning}.\hskip 1em plus 0.5em minus 0.4em\relax PMLR, 2020, pp. 86--99.

\bibitem{yuan2021agentformer}
Y.~Yuan, X.~Weng, Y.~Ou, and K.~Kitani, ``Agentformer: Agent-aware transformers for socio-temporal multi-agent forecasting,'' 2021.

\bibitem{kim2021lapred}
B.~Kim, S.~H. Park, S.~Lee, E.~Khoshimjonov, D.~Kum, J.~Kim, J.~S. Kim, and J.~W. Choi, ``Lapred: Lane-aware prediction of multi-modal future trajectories of dynamic agents,'' in \emph{CVPR}, 2021, pp. 14\,636--14\,645.

\bibitem{zhong2022stgm}
Z.~Zhong, Y.~Luo, and W.~Liang, ``Stgm: Vehicle trajectory prediction based on generative model for spatial-temporal features,'' \emph{IEEE Transactions on Intelligent Transportation Systems}, vol.~23, no.~10, pp. 18\,785--18\,793, 2022.

\bibitem{gilles2022gohome}
T.~Gilles, S.~Sabatini, D.~Tsishkou, B.~Stanciulescu, and F.~Moutarde, ``Gohome: Graph-oriented heatmap output for future motion estimation,'' in \emph{ICRA}.\hskip 1em plus 0.5em minus 0.4em\relax IEEE, 2022, pp. 9107--9114.

\bibitem{xu2023context}
P.~Xu, J.-B. Hayet, and I.~Karamouzas, ``Context-aware timewise vaes for real-time vehicle trajectory prediction,'' \emph{IEEE Robotics and Automation Letters}, 2023.

\bibitem{ren2024emsin}
Y.~Ren, Z.~Lan, L.~Liu, and H.~Yu, ``Emsin: Enhanced multi-stream interaction network for vehicle trajectory prediction,'' \emph{IEEE Transactions on Fuzzy Systems}, 2024.

\bibitem{liao2025cot}
H.~Liao, H.~Kong, B.~Wang, C.~Wang, W.~Ye, Z.~He, C.~Xu, and Z.~Li, ``Cot-drive: Efficient motion forecasting for autonomous driving with llms and chain-of-thought prompting,'' \emph{IEEE Transactions on Artificial Intelligence}, 2025.

\bibitem{zhang2024seflow}
Q.~Zhang, Y.~Yang, P.~Li, O.~Andersson, and P.~Jensfelt, ``Seflow: A self-supervised scene flow method in autonomous driving,'' in \emph{European Conference on Computer Vision}.\hskip 1em plus 0.5em minus 0.4em\relax Springer, 2024, pp. 353--369.

\bibitem{gupta2018social}
A.~Gupta, J.~Johnson, L.~Fei-Fei, S.~Savarese, and A.~Alahi, ``Social gan: Socially acceptable trajectories with generative adversarial networks,'' in \emph{CVPR}, 2018, pp. 2255--2264.

\bibitem{deo2018convolutional}
N.~Deo and M.~M. Trivedi, ``Convolutional social pooling for vehicle trajectory prediction,'' in \emph{Proceedings of the IEEE conference on computer vision and pattern recognition workshops}, 2018, pp. 1468--1476.

\bibitem{messaoud2019non}
K.~Messaoud, I.~Yahiaoui, A.~Verroust-Blondet, and F.~Nashashibi, ``Non-local social pooling for vehicle trajectory prediction,'' in \emph{2019 IEEE Intelligent Vehicles Symposium (IV)}.\hskip 1em plus 0.5em minus 0.4em\relax IEEE, 2019, pp. 975--980.

\bibitem{messaoud2021attention}
------, ``Attention based vehicle trajectory prediction,'' \emph{IEEE Transactions on Intelligent Vehicles}, vol.~6, no.~1, pp. 175--185, 2021.

\bibitem{xie2021congestion}
X.~Xie, C.~Zhang, Y.~Zhu, Y.~N. Wu, and S.-C. Zhu, ``Congestion-aware multi-agent trajectory prediction for collision avoidance,'' in \emph{ICRA}, 2021.

\bibitem{wang2023wsip}
R.~Wang, S.~Wang, H.~Yan, and X.~Wang, ``Wsip: wave superposition inspired pooling for dynamic interactions-aware trajectory prediction,'' in \emph{AAAI}, vol.~37, 2023, pp. 4685--4692.

\bibitem{zhao2019multi}
T.~Zhao, Y.~Xu, M.~Monfort, W.~Choi, C.~Baker, Y.~Zhao, Y.~Wang, and Y.~N. Wu, ``Multi-agent tensor fusion for contextual trajectory prediction,'' in \emph{Proceedings of the IEEE/CVF Conference on Computer Vision and Pattern Recognition}, 2019, pp. 12\,126--12\,134.

\bibitem{chen2022vehicle}
X.~Chen, H.~Zhang, F.~Zhao, Y.~Cai, H.~Wang, and Q.~Ye, ``Vehicle trajectory prediction based on intention-aware non-autoregressive transformer with multi-attention learning for internet of vehicles,'' \emph{IEEE Transactions on Instrumentation and Measurement}, vol.~71, pp. 1--12, 2022.

\bibitem{zuo2023trajectory}
Z.~Zuo, X.~Wang, S.~Guo, Z.~Liu, Z.~Li, and Y.~Wang, ``Trajectory prediction network of autonomous vehicles with fusion of historical interactive features,'' \emph{IEEE Transactions on Intelligent Vehicles}, 2023.

\bibitem{cong2023dacr}
P.~Cong, Y.~Xiao, X.~Wan, M.~Deng, J.~Li, and X.~Zhang, ``Dacr-amtp: Adaptive multi-modal vehicle trajectory prediction for dynamic drivable areas based on collision risk,'' \emph{IEEE Transactions on Intelligent Vehicles}, 2023.

\bibitem{liao2024human}
H.~Liao, S.~Liu, Y.~Li, Z.~Li, C.~Wang, Y.~Li, S.~E. Li, and C.~Xu, ``Human observation-inspired trajectory prediction for autonomous driving in mixed-autonomy traffic environments,'' in \emph{ICRA}.\hskip 1em plus 0.5em minus 0.4em\relax IEEE, 2024, pp. 14\,212--14\,219.

\bibitem{bengio2013estimating}
Y.~Bengio, N.~L{\'e}onard, and A.~Courville, ``Estimating or propagating gradients through stochastic neurons for conditional computation,'' \emph{arXiv preprint arXiv:1308.3432}, 2013.

\end{thebibliography}

%
% <OR> manually copy in the resultant .bbl file
% set second argument of \begin to the number of references
% (used to reserve space for the reference number labels box)
%\begin{thebibliography}{1}

%\bibitem{IEEEhowto:kopka}
%H.~Kopka and P.~W. Daly, \emph{A Guide to \LaTeX}, 3rd~ed.\hskip 1em plus
%  0.5em minus 0.4em\relax Harlow, England: Addison-Wesley, 1999.

%\end{thebibliography}

% biography section
% 
% If you have an EPS/PDF photo (graphicx package needed) extra braces are
% needed around the contents of the optional argument to biography to prevent
% the LaTeX parser from getting confused when it sees the complicated
% \includegraphics command within an optional argument. (You could create
% your own custom macro containing the \includegraphics command to make things
% simpler here.)
%\begin{IEEEbiography}[{\includegraphics[width=1in,height=1.25in,clip,keepaspectratio]{mshell}}]{Michael Shell}
% or if you just want to reserve a space for a photo:

\begin{IEEEbiography}
[{\includegraphics[width=1in,height=1.25in, clip,keepaspectratio]{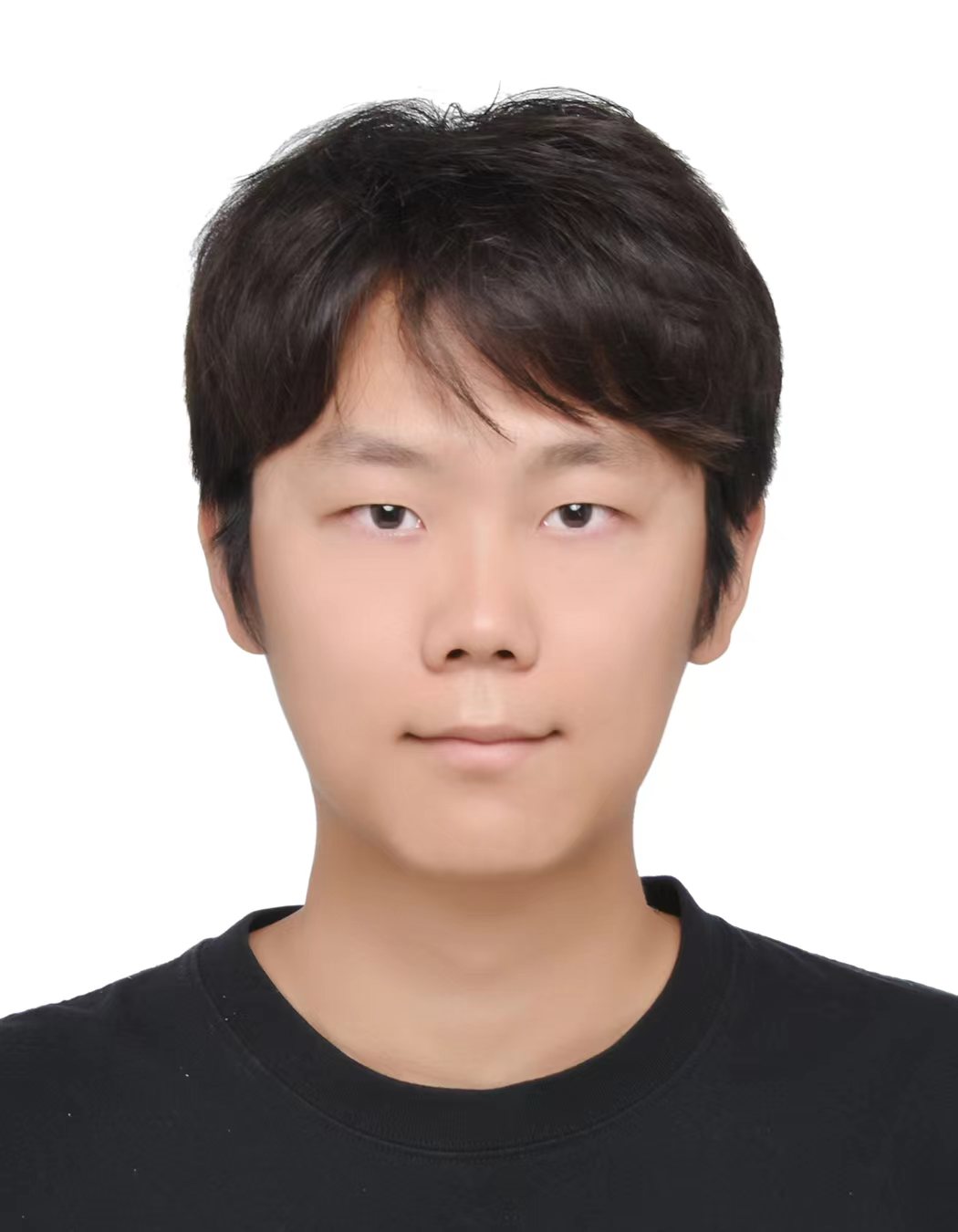}}]{Chengyue Wang} is currently pursuing a Ph.D. degree at the State Key Laboratory of Internet of Things for Smart City and the Department of Civil Engineering, University of Macau. He received his M.S. degree from the University of Illinois Urbana-Champaign (UIUC) in 2022. He received his B.E. degree from Chang'an University in 2021. His research interests include autonomous driving and intelligent transportation systems.
\end{IEEEbiography}

\begin{IEEEbiography}
[{\includegraphics[width=1in,height=1.25in, clip,keepaspectratio]{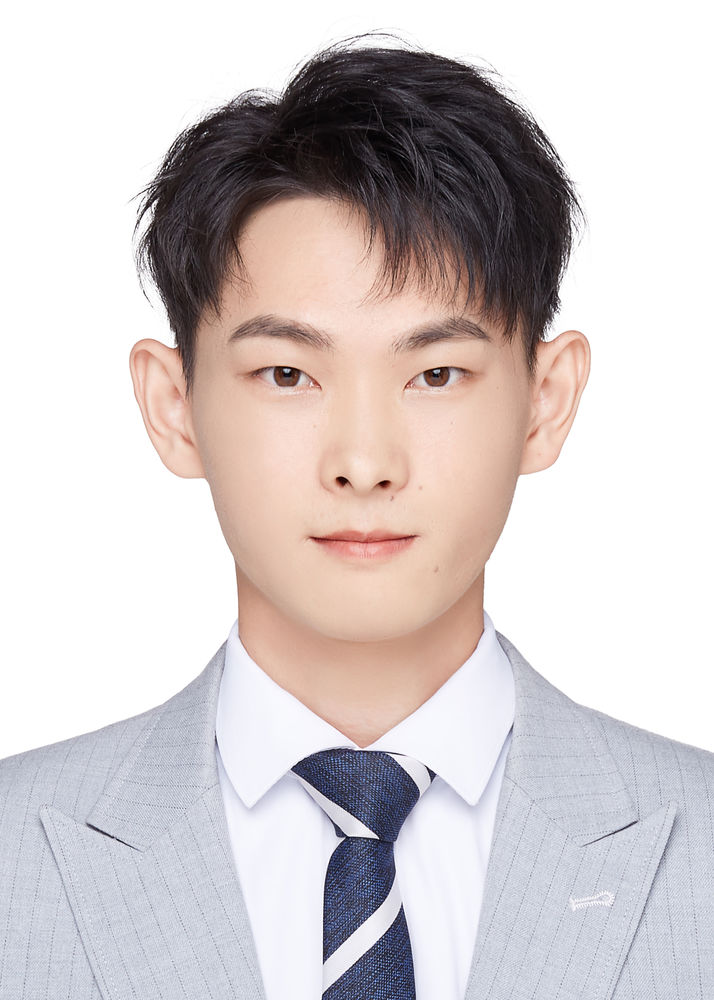}}]{Bin Rao} received the B.S. and M.S. degrees in traffic engineering from the South China University of Technology, Guangzhou, China, in 2021 and 2024. He is currently working as a research assistant at the State Key Laboratory of Internet of Things for Smart City, University of Macau. His research interests include transportation network modeling, time series prediction, autonomous driving trajectory prediction and planning, and accident anticipation.
\end{IEEEbiography}

\begin{IEEEbiography}
[{\includegraphics[width=1in,height=1.25in, clip,keepaspectratio]{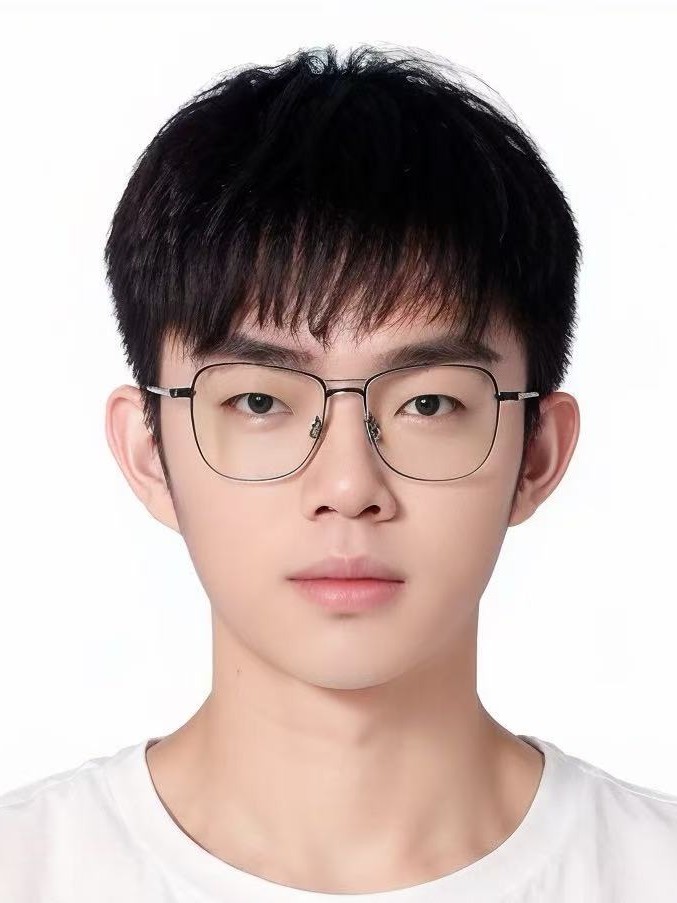}}]{Haicheng Liao} (Student Member, IEEE) received the B.S. degree in software engineering from the University of Electronic Science and Technology of China (UESTC) in 2022. He is currently pursuing the Ph.D. degree at the State Key Laboratory of Internet of Things for Smart City and the Department of Computer and Information Science, University of Macau. Over his academic career, he has published over 20 papers. His research interests include connected autonomous vehicles and the application of deep reinforcement learning to autonomous driving.
\end{IEEEbiography}

\begin{IEEEbiography}
[{\includegraphics[width=1in,height=1.25in, clip,keepaspectratio]{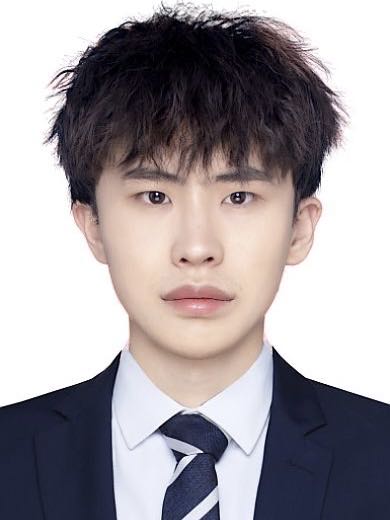}}]{Bonan Wang} is currently pursuing an M.S. degree at the State Key Laboratory of Internet of Things for Smart City and the Department of Computer and Information Science, University of Macau. He received a B.S. degree in Data Science and Big Data Technology from Shaanxi University of Science \& Technology in 2023. His research interests include trajectory prediction and planning for autonomous driving.
\end{IEEEbiography}

\begin{IEEEbiography}
[{\includegraphics[width=1in,height=1.25in,clip,keepaspectratio]{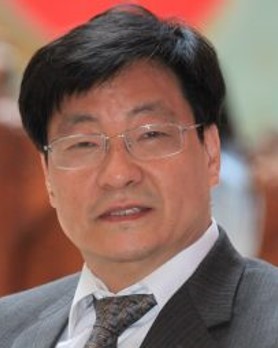}}]{Chengzhong Xu} (Fellow, IEEE) received the Ph.D. degree from The University of Hong Kong, in 1993. He is currently the chair professor of computer science and the dean with the Faculty of Science and Technology, University of Macau. Prior to this, he was with the faculty at Wayne State University, USA, and the Shenzhen Institutes of Advanced Technology, Chinese Academy of Sciences, China. He has published more than 400 papers and more than 100 patents. His research interests include cloud computing and data-driven intelligent applications. He was the Best Paper awardee or the Nominee of ICPP2005, HPCA2013, HPDC2013, Cluster2015, GPC2018, UIC2018, and AIMS2019. He also won the Best Paper award of SoCC2021. He was the Chair of the IEEE Technical Committee on Distributed Processing from 2015 to 2019.
\end{IEEEbiography}

\begin{IEEEbiography}
[{\includegraphics[width=1in,height=1.25in,clip,keepaspectratio]{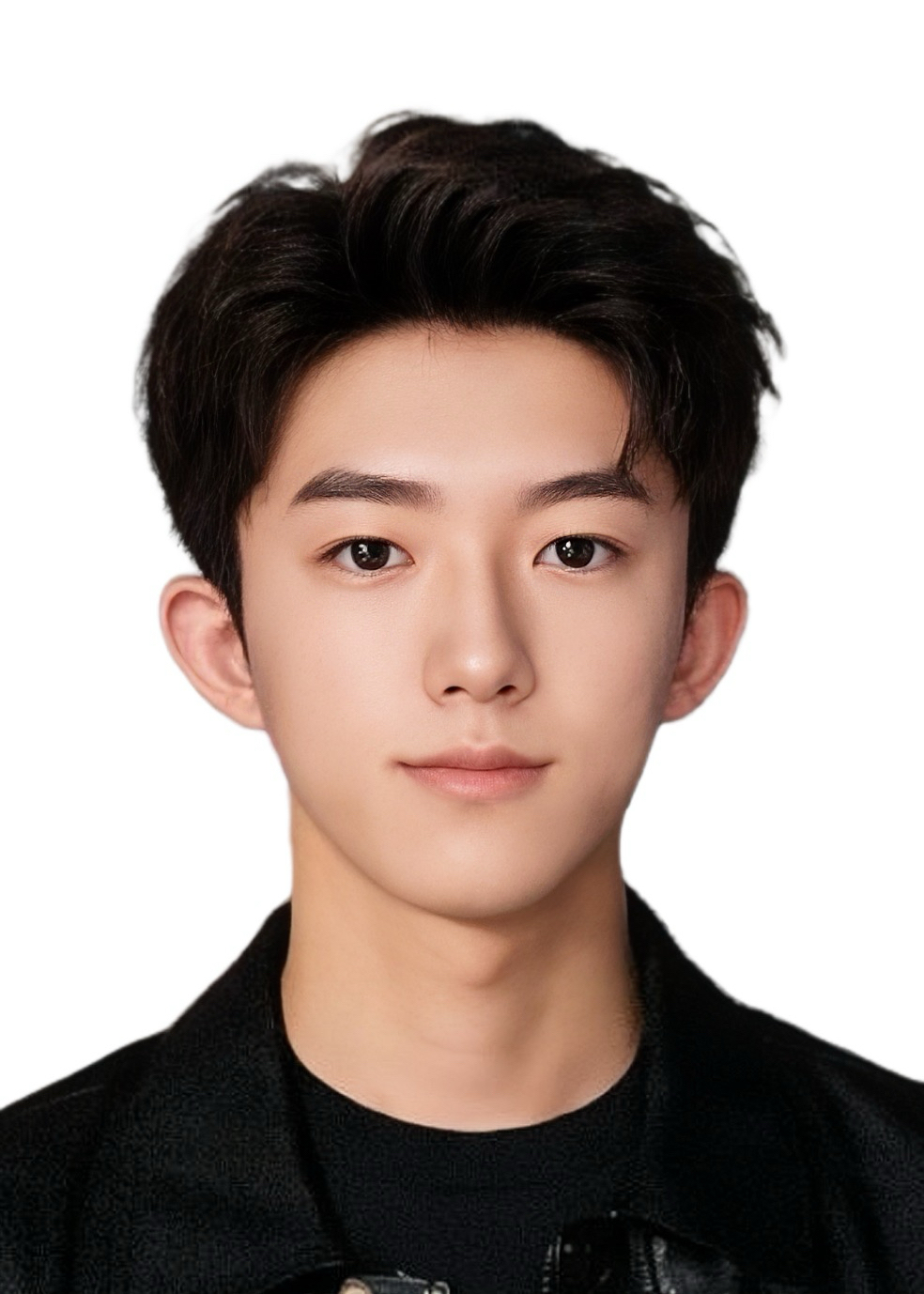}}] {Zhenning Li} (Member, IEEE) received his Ph.D. in Civil Engineering from the University of Hawaii at Manoa, Honolulu, Hawaii, USA, in 2019. Currently, he holds the position of Assistant Professor at the State Key Laboratory of Internet of Things for Smart City, as well as the Department of Computer and Information Science at the University of Macau, Macau. Over his academic career, he has published over 80 papers. His main areas of research focus on the intersection of connected autonomous vehicles and Big Data applications in urban transportation systems. He has been honored with several awards, including the TRB best young researcher award and the CICTP best paper award.
\end{IEEEbiography}

% You can push biographies down or up by placing
% a \vfill before or after them. The appropriate
% use of \vfill depends on what kind of text is
% on the last page and whether or not the columns
% are being equalized.

%\vfill

% Can be used to pull up biographies so that the bottom of the last one
% is flush with the other column.
%\enlargethispage{-5in}

% that's all folks
\end{document}